\documentclass[nohyperref]{article}

\usepackage{hyperref}

\usepackage[accepted]{icml2023}

\usepackage{amsmath}
\usepackage{amssymb}
\usepackage{mathtools}
\usepackage{amsthm}

\usepackage[capitalize,noabbrev]{cleveref}

\theoremstyle{plain}
\newtheorem{theorem}{Theorem}[section]
\newtheorem{proposition}[theorem]{Proposition}
\newtheorem{lemma}[theorem]{Lemma}

\theoremstyle{definition}
\newtheorem{definition}[theorem]{Definition}

\theoremstyle{remark}

\usepackage{comment}

\usepackage[utf8]{inputenc} 
\usepackage[T1]{fontenc}    
\usepackage{url}            
\usepackage{graphicx}
\usepackage{subfigure}
\usepackage{booktabs}       
\usepackage{amsfonts}       
\usepackage{nicefrac}       
\usepackage{microtype}      
\usepackage{xcolor}         
\usepackage{kky}
\usepackage[symbol]{footmisc}

\renewcommand{\thefootnote}{\fnsymbol{footnote}}

\usepackage{wrapfig}
\usepackage{algorithmic}
\usepackage{algorithm}

\usepackage{dsfont}
\usepackage{bbm}

\newcommand{\CR}{\operatorname{CR}}
\newcommand{\ALG}{\operatorname{ALG}}
\newcommand{\FTPP}{\operatorname{FTPP}}

\newcommand{\OPT}{\operatorname{OPT}}
\newcommand{\ETCU}{\operatorname{ETC-U}}
\newcommand{\UCBU}{\operatorname{UCB-U}}
\newcommand{\LSEPT}{\operatorname{LSEPT}}
\newcommand{\ETCRR}{\operatorname{ETC-RR}}
\newcommand{\UCBRR}{\operatorname{UCB-RR}}
\newcommand{\RRR}{\operatorname{RR}}

\newcommand{\lambdab}{\boldsymbol{\lambda}}

\usepackage{scrextend}

\def\showComments{} 

\ifdefined\showComments
    \newcommand{\flo}[1]{{\color{teal} {\{Flore: #1\}}}}
    \newcommand{\nadav}[1]{{\color{magenta} {\{Nadav: #1\}}}}
    \newcommand{\addref}{{\color{red}(reference needed)}}
    \newcommand{\hr}[1]{{\color{blue} {\{Hugo: #1\}}}}
\else
    \newcommand{\flo}[1]{}
    \newcommand{\nadav}[1]{}
    \newcommand{\addref}[1]{}
    \newcommand{\hr}[1]{}
\fi

\usepackage{selectp}

\usepackage[capitalize,noabbrev]{cleveref}

\usepackage[textwidth=\marginparwidth, textsize=tiny]{todonotes}

\icmltitlerunning{On Preemption and Learning in Stochastic Scheduling}

\begin{document}

\twocolumn[
\icmltitle{On Preemption and Learning in Stochastic Scheduling}



\icmlsetsymbol{equal}{*}

\begin{icmlauthorlist}
\icmlauthor{Nadav Merlis}{equal,ensae}
\icmlauthor{Hugo Richard}{equal,ensae,criteo}
\icmlauthor{Flore Sentenac}{equal,ensae}
\icmlauthor{Corentin Odic}{ensae}
\icmlauthor{Mathieu Molina}{ensae,inria}
\icmlauthor{Vianney Perchet}{ensae,criteo}
\end{icmlauthorlist}

\icmlaffiliation{ensae}{ENSAE, Paris}
\icmlaffiliation{criteo}{Criteo, Paris}
\icmlaffiliation{inria}{Inria, France}

\icmlcorrespondingauthor{Nadav Merlis}{nadav.merlis@ensae.fr}

\icmlkeywords{Machine Learning, ICML}

\vskip 0.3in
]



\printAffiliationsAndNotice{\icmlEqualContribution} 

\begin{abstract}
    We study single-machine scheduling of jobs, each belonging to a job type that determines its duration distribution. We start by analyzing the scenario where the type characteristics are known and then move to two learning scenarios where the types are unknown: non-preemptive problems, where each started job must be completed before moving to another job; and preemptive problems, where job execution can be paused in the favor of moving to a different job. In both cases, we design algorithms that achieve sublinear excess cost, compared to the performance with known types, and prove lower bounds for the non-preemptive case. Notably, we demonstrate, both theoretically and through simulations, how preemptive algorithms can greatly outperform non-preemptive ones when the durations of different job types are far from one another, a phenomenon that does not occur when the type durations are known.
\end{abstract}
\section{Introduction}

Single Machine Scheduling is a longstanding problem with many variants and applications \cite{pinedo2012scheduling}. In this problem, a set of $N$ jobs must be processed on one machine, each of a different `size' -- processing time required for its completion. An algorithm is a policy assigning jobs to the machine, and performance is usually measured by \emph{flow time} -- the sum of the times when jobs have finished. 
If one has access to the size of each job, then scheduling the jobs by increasing size is optimal~\cite{schrage1968proof}. Unfortunately, for most applications, this knowledge is unavailable; yet, oftentimes, some structure or knowledge on the jobs can still be leveraged.

In this paper, we focus on scheduling problems where jobs are grouped by types that determine their duration distribution. This model approximates many real-world scenarios. For example, when scheduling patients for surgery, patients may be grouped by expected procedure time~\cite{magerlein1978surgical}. The model is also relevant in computing problems, where jobs with similar features are expected to have a similar processing time~\cite{li2006improving}.
Lastly, in calendar learning, where an agent advises the user on how to organize its day based on the tasks to be done, similar tasks can be assumed to have a similar duration~\cite{white2019task}.

In practice, when encountering a new scheduling task, we usually know the type of each job, but have little-to-no information on the expected duration under each type. Then, the scheduling algorithm must \emph{learn} the characteristics of each type to be able to utilize this information. This must be done concurrently with the scheduling of tasks, which poses an extra challenge -- to be useful, learning must be done as early as possible; however, wrong scheduling allocation at the beginning delays all jobs and causes large penalties.

In this work, we show how learning can be efficiently done in scheduling problems with job types, characterized by exponential distributions, in two different settings -- the non-preemptive setting, where once a job started running, it must be completed, and the preemptive setting, where jobs can be put on hold. We present two algorithms in each setting and show that the preemptive setting has a clear advantage when the type durations have to be learned. This comes in contrast to the case of known types, where under reasonable assumptions, the optimal algorithm is non-preemptive.

While our algorithms resemble classic bandit methods, the scheduling objective requires different analysis approaches. In particular, in the context of scheduling, the quality of an algorithm is measured by the \emph{ordering} of jobs. In stark contrast, regret-minimization objectives measure the \emph{number of plays} from each arm (job type). Indeed, in scheduling problems, the number of pulls from each job type is always the same -- by the end of the interaction, we would finish all the jobs of all types. Thus, both our algorithmic design and analysis will be comparative -- focus on the number of jobs evaluations from a bad type before the completion of jobs of a good type.

Our contributions are as follows. \textbf{(1)} We present the scheduling setting with unknown job types. \textbf{(2)} We analyze the optimal algorithm for the case of known job types, called Follow-the-Perfect-Prediction ($\FTPP$), and bound its competitive ratio (CR). \textbf{(3)} We present explore-then-commit (ETC) and upper confidence bound (UCB) algorithms for the preemptive and non-preemptive settings and bound their performance, compared to $\FTPP$. In particular, our bounds show that the non-preemptive algorithms have worse dependence on the durations of the longest job types. \textbf{(4)} We complement this by proving lower bounds to the non-preemptive case.
\textbf{(5)} We end by simulating our suggested algorithms and show that their empirical behavior is consistent with our theoretical findings.

\section{Related Work}

\textbf{Scheduling problems.} The scheduling literature and problem zoology are large. We focus on static scheduling on a single machine with the objective of minimizing flow time. Static scheduling~\cite{motwani1994nonclairvoyant} means that all scheduled tasks are given in advance before the scheduling starts. Possible generalizations include dynamic scheduling where scheduled tasks arrive online~\cite{becchettiNonclairvoyantSchedulingMinimize2004}, weighted flow time~\cite{bansalMinimizingWeightedFlow2007} where different jobs have different weights,
multiple machines~\cite{lawler1978preemptive}); and many more~\cite{durr2020adversarial, tsung1997optimal}. While we only tackle some versions of this problem, we believe that our approaches can be adapted or extended to other settings.

\emph{Clairvoyant and non-clairvoyant scheduling.} In clairvoyant scheduling, job sizes are assumed to be known, and scheduling the shortest jobs first gives the lowest flow time~\cite{schrage1968proof}. In non-clairvoyant scheduling, job sizes are arbitrary and unknown. The Round Robin (RR) algorithm, which gives the same amount of computing time to all jobs, is the best deterministic algorithm with a competitive ratio of $2 - \frac{2}{N+1} = 2 + o(N)$~\cite{motwani1994nonclairvoyant}. The best randomized algorithm has a competitive ratio of $2 - \frac{4}{N+3} = 2 + o(N)$~\cite{motwani1994nonclairvoyant}. 

\emph{Stochastic scheduling.} Stochastic scheduling covers a middle ground where job sizes are known random variables. The field of \emph{optimal stochastic scheduling} aims to design optimal algorithms for stochastic scheduling (see~\citealt{caiOptimalStochasticScheduling2014} for a review). When distributions have a non-decreasing hazard rate, scheduling the shortest mean first is optimal (see~\citealt{caiOptimalStochasticScheduling2014}, Corollary 2.1). 

In this work, we consider exponential job sizes (which have a non-decreasing hazard rate), as frequently assumed in the scheduling literature~\cite{kampke1989optimal, hamada1993bayesian, cunningham1973scheduling, cai2000asymmetric, cai2005single, pinedo1985scheduling, glazebrook1979scheduling} and similarly for the presence of different types of jobs~\cite{mitzenmacher2020scheduling, hamada1993bayesian, marban2011learning}. Yet, in contrast to most of the literature on stochastic scheduling, the means of the exponential sizes are unknown to the scheduler and are learned as the algorithm runs. Nonetheless, we later present algorithms whose CR asymptotically converges to the optimal value, obtained in stochastic scheduling with known job means.

The problem of learning in scheduling has received some attention lately. Specifically, \citet{levi2019scheduling} consider a setting where it is possible to `test' jobs to learn about their attributes, which comes at a cost. In \cite{krishnasamy2018learning}, the authors propose an algorithm to learn the $c\mu$ rule (a rule to balance different holding costs per job) in the context of dynamic queues. Perhaps closest to our setting, in \cite{lee2021scheduling}, job types are also considered, but the length of the jobs is assumed to be known, and the goal is to deal with the uncertainty on the holding costs, which are noisily observed at each iteration. In the last two papers, no explicit exploration is needed, which stands in contrast with our setting. 

The problem we tackle was previously studied in a Bayesian setting~\cite{marban2011learning}, under the assumption of two job types, and a Bayesian algorithm, called $\LSEPT$, was presented. When run with an uninformative prior (the same for all job types), $\LSEPT$ is reduced to a greedy algorithm; whenever a job finishes, it runs until completion a job whose type has the lowest expected belief on its mean size (computed across jobs that have been processed so far). The author proved it has better performance in expectation than fully non-adaptive methods, but provided no other guarantee. In \Cref{app:comparisonLSPET}, we empirically evaluate this algorithm and show it has a behavior typical of greedy algorithms: it has a very large variance, and its CR does not converge to the optimal CR, in contrast to our suggested methods.


\renewcommand*{\thefootnote}{\arabic{footnote}}

\section{Setting and Notations}

We consider scheduling problems of $N$ jobs on a single machine, each belonging to one of $K$ job types. 
We assume that $N=nK$, i.e, there are $n$ jobs of each type. The different sizes (also called processing times) of the jobs of type $k$ are denoted $(P^{k}_i)_{i \in [n]}$,\footnote{For clarity of exposition, we assume that there are exactly $n$ jobs per type. When types have different numbers of jobs $n_1,n_2,\ldots,n_K$, all algorithms can run with $n = \max_{\ell}n_\ell$, and all bounds hold with this same parameter $n$.} where $P^{k}_i \sim \Ecal(\lambda_k)$ are independent samples from an exponential variable of parameter $\lambda_k$. 
By extension, $\EE[P_i^k]= \lambda_k$, and we call $\lambda_k$ the mean size of type $k$. 
We assume without loss of generality that the mean sizes of the $K$ types are in an increasing order $\lambda_1\leq \lambda_2 \ldots \leq \lambda_{K}$ and denote $\lambdab = (\lambda_1, \dots, \lambda_K)$. With slight abuse of notations, we sometimes ignore the job types and denote the job durations by $P_i$ for $i\in[N]$. 

Next, denote $b_i^{k}$ and $e_i^{k}$ the beginning and end dates of the computation of the $i^{th}$ job of type $k$. We define the cost of an algorithm $\ALG$, also called \emph{flow time}, as the sum of all completion times: $C_{{\ALG}} = \sum_{k=1}^K\sum_{i \in n}e_i^k$. Given knowledge of the job size realizations, the cost is minimized by an algorithm that computes them in increasing order, which we term as $\OPT$.  

\emph{Preemption} is the operation of pausing the execution of one job in the favor of running another one. Thus, \emph{preemptive algorithms} are ones that support preemption, while \emph{non-preemptive algorithms} do not allow it and must run each started job until completion.


\section{Benchmark: Follow The Perfect Prediction}
We compare our algorithms to a baseline that completes each job by increasing expected sizes, called \emph{Follow-The-Perfect-Prediction} ($\FTPP$). For exponential job sizes, this strategy is optimal between all algorithms without access to job size realizations (see~\citealt{caiOptimalStochasticScheduling2014}, Corollary 2.1).  Thus, with learning, we aim at designing algorithms approaching the performance of $\FTPP$, whilst mitigating the cost of learning (i.e., mitigating the cost of exploration). In the rest of this section, we analyze the performance of $\FTPP$. 

First, we evaluate the performance of non-clairvoyant algorithms that do not exploit the job type structure. We then compare the performance of the best of those algorithms against that of $\FTPP$ and show the clear advantage of using the structure of job types.

\subsection{Non-clairvoyant Algorithms} An algorithm $A$ is said non-clairvoyant if it does not have any information on the job sizes, including the job type structure.  Recall that $\RRR$ is the algorithm that computes all unfinished jobs in parallel and is the optimal deterministic algorithm in the adversarial setting. The following proposition states that in our setting, it is the optimal algorithm among all non-clairvoyant ones.

\begin{proposition}
	\label{prop:lowerbound}
	For any $\lambdab$ and any (deterministic or randomized) non-clairvoyant algorithm $A$, there exists a job ordering such that $\EE[C_A] \geq \EE[C_{\RRR}]$. 
\end{proposition}
\begin{proof}[Proof sketch (full proof in \Cref{app:lowerbound})]
Consider $T_{ij}^A$, the amount of time that job $i$ and job $j$ delay each other. As the algorithm is unaware of the expected job size order, its run is independent of whether the expected size of job $i$ is smaller or greater than that of $j$. This holds because a non-clairvoyant algorithm has no information on job expected sizes nor on the existence of job types. As an adversary, we can therefore choose the job order so that the algorithm incurs the largest flow time. 
A careful analysis then provides $\EE[T_{ij}^A] \geq 2 \EE[T_{ij}^{OPT}]$ where OPT is the optimal realization-aware algorithm. A similar reasoning is made in the case of randomized algorithms. We conclude by observing that $\RRR$ achieves such delay.
\end{proof}

Unfortunately, even though our setting is not adversarial, the CR of $\RRR$ is bounded from below (\Cref{lemma:CR RR}):
\begin{equation}\label{eq:CR_RRR}
  \text{for any $\lambdab$, }\quad  \frac{\EE\left[C_{\RRR}\right]}{\EE\left[C_{\OPT}\right]}\geq 2-\frac{4}{n+3}.
\end{equation}

\subsection{Performance of FTPP}

The first statement establishes that $\FTPP$ outperforms $\RRR$ on any instance.

\begin{lemma}
For any $n$ and $\lambdab$,
\[
\EE[C_{\FTPP}] \leq \EE[C_{\RRR}].
\]
\end{lemma}

The proof is a straightforward computation, done in \Cref{app:FTPPvsRR}. 

This indicates that when information on the job types is available, it is always advantageous to use it. In the rest of the section, we quantify the improvement this extra information brings. More precisely, we show that on a wide variety of instances, the CR of $\FTPP$ is much smaller than that of $\RRR$. We first present such a bound when $K=2$.
\begin{proposition}
\label{prop:CR-FTPP-2types}
The CR of $\FTPP$ with $K=2$ types of jobs with $n$ jobs per type with $\lambda_1 = 1$ and $\lambda_2 = \lambda > 1$ satisfies:
\begin{align*}
\frac{\EE[C_{\FTPP}]}{\EE[C_{\OPT}]} \leq 2 - 4\frac{\lambda-1}{{(1 + \lambda)^2 + 4 \lambda}}.
\end{align*}
\end{proposition}
\begin{proof}[Proof sketch (full proof in \Cref{prop:CR-FTPP-Ktypes})]
  The total expected flow time of any algorithm is given by the sum of the expected time spent computing all jobs and the expected time lost waiting as jobs delay each other.
  In the case of $\OPT$, the expected flow time can be computed in closed form using that job $i$ and $j$ delay each other by $\EE[\min(X_i, X_j)] = \frac{\EE[X_i] \EE[X_j]}{\EE[X_i] + \EE[X_j]}$.
  The CR $\frac{\EE[C_{\FTPP}]}{\EE[C_{\OPT}]}$ can then be calculated and is upper bounded to yield the result.
\end{proof}

In the case $K=2$, there exists values of $\lambda$ for which the CR of FTPP is lower than $1.71$. In the general case, \Cref{prop:CR-FTPP-Ktypes3} in \Cref{sec:CR-FTTP-Ktypes3} shows that there exist values of $K$ and $\lambdab$ for which the CR is as low as $1.274$.\footnote{An exact expression for the CR of $\FTPP$ is given at \Cref{eq:cr:ftpp} in the appendix and is omitted for clarity reasons.}


\section{Non-Preemptive Algorithms}

After establishing $\FTPP$ as the baseline for learning algorithms, we move to tackle learning in the non-preemptive setting, where once started, job execution cannot be stopped (see \Cref{algo:non-preemptive}). This is relevant, for example, to settings where switching tasks is very costly (e.g., in running time or memory) or even impossible (e.g., in medical applications, where treatment of a patient cannot be stopped). We show how algorithms from the bandit literature can be adapted to the scheduling setting and bound their excessive cost, compared to $\FTPP$. Specifically, by treating each job type as an `arm', we adapt explore-then-commit and optimism-based strategies to the scheduling setting. 



\begin{algorithm}[t]
	\caption{Non-Preemptive Algorithms routine}
	\label{algo:non-preemptive}
\begin{algorithmic}[1]
\STATE \textbf{Init}: type set $\Ucal=[K]$, active jobs $i_k=1,\forall k\in[K]$
\WHILE{$\Ucal$ is not empty}
\STATE Use a type selection subroutine to select a type $k\in\Ucal$
\STATE Run job $i_k$ until completion
\STATE Set $i_k\gets i_k+1$
 \IF{All jobs of type $k$ are completed}
	\STATE Remove type $k$ from $\Ucal$
 \ENDIF
 \ENDWHILE
 \end{algorithmic}
\end{algorithm}

\subsection{Description of ETC-U and UCB-U}
In the following, we describe the type selection mechanism for $\ETCU$ and $\UCBU$. The full pseudo-code of both algorithms is available in \Cref{app: non-preemptive algs}.

Let $\Ucal$ be the set of all job types with at least one remaining job.

\paragraph{ETC-U type selection.} While $\ETCU$ runs, it maintains a set of types $\Acal$ that are candidates for having the lowest mean size among the incomplete types $\Ucal$. 
At each iteration, $\ETCU$ chooses a job of type $k$ of the minimal number of completed jobs in $\Acal$ and executes it to completion. Then, $\Ucal$ and $\Acal$ are updated and the procedure repeats until no more jobs are available in $\Ucal$.

We now describe the mechanism of maintaining the candidate type set $\Acal$. At a given iteration, denote by $m_k$ and $m_\ell$, the number of jobs of type $k$ and $\ell$ that have been computed up to that iteration. Letting
\begin{align*}
    &\hat{r}_{k, \ell}^{\min(m_{k}, m_{\ell})} =  \frac{\sum_{i=1}^{\min(m_{k}, m_{\ell})} \ind{P^k_i < P^\ell_i}}{\min(m_{k}, m_{\ell})}  \quad \text{and} \\
    &\delta_{k, \ell}^{\min(m_{k}, m_{\ell})} = \sqrt{\frac{\log(2n^2K^3)}{2 \min(m_{k}, m_{\ell})}},
\end{align*}
a type $\ell$ is excluded from $\Acal$ if there exists a type $k$ such that
\begin{align}\label{eq:elimination_condition}
    \hat{r}_{k, l}^{\min(m_{k}, m_{\ell})}-\delta_{k, \ell}^{\min(m_{k}, m_{\ell})}>0.5.
\end{align}
In the proof, we show that this condition implies w.h.p. that $\lambda_{k}<\lambda_\ell$. Thus, when it holds, job type $\ell$ is no longer a candidate for the remaining job type with the smallest expectation, and we say that type $k$ eliminates type $\ell$. Once a job type is eliminated, it remains so until $\mathcal{A}$ is empty, at which point all job types in $\mathcal{U}$ are reinstated to $\mathcal{A}$.

Finally, whenever $\Acal$ contains only one type $k$, all jobs of this type are run to completion, and after all jobs from type $k$ are finished, it is removed from $\Ucal$ and therefore from $\Acal$. This means that types that were eliminated by $k$ can be candidates again.

\paragraph{UCB-U type selection.} At every iteration, the algorithm computes an index for each job type and plays a type with the minimal index from the incomplete types $\Ucal$. Specifically, if $m_k$ jobs were completed from type $k$, the index of the type is defined as
\begin{equation*}
\underline{\lambda}_k^{m_k} = \frac{2 \sum_{i=1}^{m_k} X_i^k}{\chi^2_{2 m_k}(1 - \frac{1}{2 n^2K^2})},
\end{equation*}
where $\chi^2_{m}(\delta)$ is the $\delta$-percentile of a $\chi^2$ distribution with $m$ degrees of freedom. 
In the proof, we show that these indices are a lower bound of the job means w.h.p., so choosing the minimal index corresponds to choosing the type with the optimistic shortest duration.

\subsection{Cost Analysis}

\begin{proposition}\label{prop:bound-nonpreemptive}
The following bounds hold:
     {\small
\begin{align*}
&\EE[C_{\ETCU}] 
    \leq \EE[C_{\FTPP}] +\frac{1}{n}\EE[C_{\OPT}]\\
    &+\!\!\sum_{k\in [K]}\!\left[\frac{1}{2}(k-1)(2K-k)+(K-k)^2\right]\!\lambda_k n \sqrt{8n \log(2 n^2K^3)}\\
\end{align*}}
and
\small{
\begin{align*}
    \EE[C_{\UCBU}] 
    \leq& \EE[C_{\FTPP}]+ \frac{2}{n}\EE[C_{\OPT}] \\
    &+ n (K-1)\sqrt{3n\ln\rbr{2n^2K^2}} \sum_{k=1}^K\lambda_k. 
\end{align*}}
\end{proposition}
\Cref{prop:bound-nonpreemptive} shows that both ETC-U and UCB-U have sublinear excess cost. Indeed the optimal cost and thus also the cost of FTPP is lower bounded by $\Omega(n^2 \sum_k \lambda_k)$ which makes the terms in $O(n\sqrt{n})$ strictly sublinear in $n$ compared to the optimal cost. When all parameters lie within a finite range bounded away from 0, the optimal cost scales as $K^2 n^2$ making the excess cost linear in $K$ (up to log terms) compared to the optimal cost.

\begin{proof}[Proof sketch (full proof in \Cref{app:non-preemptive})]
The above proposition is a concatenation of Propositions \ref{app:fullboundETCU} and \ref{app:bound_UCBU_full}.

The proof of both bounds starts with the decomposition of the cost with the following Lemma (proven in \Cref{app:cost decomp non-preemptive}).

\begin{lemma}[Cost of non-preemptive algorithms]
  \label{app:lemma:cost-non-preemptive main paper}
  Any non-preemptive algorithm $A$ has a cost
  \begin{align*}
\EE[C_{A}] = &\EE[C_{\FTPP}] \\
&+ \sum_{\substack{(\ell,k) \in [K^2], k>\ell\\(i,j) \in [n]^2}}( \lambda_k-\lambda_\ell)\EE\sbr{\ind{e_i^k \leq b_j^\ell}}.
  \end{align*}
\end{lemma}
This lemma is obtained by computing explicitly the expected cost of algorithm $\FTPP$ and using the fact that the realized length of the jobs conditioned on their type is independent of their start date.

Then the two proofs diverge.

For $\ETCU$, the first step is to prove that condition \eqref{eq:elimination_condition} implies w.h.p that $\lambda_k\leq \lambda_\ell$, which implies that the type in $\mathcal{U}$ with the smallest mean is never eliminated. Then, for the sake of the analysis, the run of the algorithm is divided into phases. In phase number $\ell$, type $\ell$ is the job type remaining with the smallest mean. We then bound the total number of samples before an arm with a large mean is eliminated at phase $\ell$.

The proof for UCB also starts by showing that w.h.p., arm indices lower bound the true means. Then, under the condition that the bounds hold, we upper bound the number of times an arm of type $k\geq \ell$ can be pulled while type $\ell$ is still active.
\end{proof}

The bounds in \Cref{prop:bound-nonpreemptive} hold for any value of the parameters. When the parameter values are far from each other, tighter bounds hold. We give here these tighter bounds for $K=2$ when $\lambda_2\geq 3\lambda_1$. A more general version of this bound is given in the appendix, Propositions \ref{app:fullboundETCU} and \ref{app:bound_UCBU_full}.

\begin{lemma}\label{lem:tighternonpreemptive}
    If $K=2$ and $\lambda_2\geq 3\lambda_1$, the following bounds hold:
    {\small
    \begin{align*}
    \EE[C_{\ETCU}] 
    \leq \EE[C_{\FTPP}] +12\lambda_2n \log(2 n^2K^3)+\frac{2}{n}\EE[C_{\OPT}],
    \end{align*}}
and
    {\small 
    \begin{align*}
    \EE[C_{\UCBU}] 
    \leq \EE[C_{\FTPP}] + \frac{9}{2}\lambda_2 n\ln\rbr{2n^2K^2} 
    + \frac{4}{n}\EE[C_{\OPT}].
    \end{align*}}
\end{lemma}

The bounds of this section seem quite discouraging --  they imply that the existence of even one type of extremely large duration has grave implications on the cost of any algorithm. Unfortunately, for any non-preemptive algorithm, an extra cost w.r.t. $\FTPP$  scaling as $n\lambda_K$ is unavoidable. Indeed, in the beginning, no information on the mean types is available, and any started job will be fully computed, delaying all remaining $nK-1$ jobs (see  \Cref{app: lower bound large diff} for a formal proof).


\subsection{Lower bound}

We end this section by analyzing lower bounds for any non-preemptive scheduling algorithm. In particular, we focus on the dependency of the excessive cost, compared to $\FTPP$, as a function of $n$. We focus on lower bounds for the case of $K=2$, providing a lower bound when $\lambda_1$ and $\lambda_2$ are close to each other and showing that in this case, the excess cost increases as $n\sqrt{n}$.
\begin{proposition}[Dependency in $n$]
\label{prop:dependency in n}
For any $\lambda_1<\lambda_2$, the flow time of any (possibly randomized) non-preemptive algorithm $A$ satisfies:
\begin{align*}
\EE[C_A] &\geq \EE[C_{\FTPP}] + (\lambda_2 - \lambda_1)n^2\frac{\exp(-n\frac{(\lambda_2 - \lambda_1)^2}{\lambda_1 \lambda_2})}{8}.
\end{align*}
In particular, for any $\lambda_2 \leq \lambda_1\rbr{1 + \frac{1}{\sqrt{n}}}$,
\begin{align*}
    \EE[C_A] &\geq \EE[C_{\FTPP}] + (\lambda_1+\lambda_2)n\sqrt{n} \frac{e^{-1/4}}{24}.
\end{align*}
\end{proposition}
\begin{proof}[Proof sketch (full proof in \Cref{app: lower bound small diff})]
We start with the decomposition of \Cref{app:lemma:cost-non-preemptive main paper} and look at the event 
\begin{align*}
    E = \ind{\sum_{(i,j) \in [n]^2} \ind{e_j^2<b_i^1} \geq n^2 / 2}.
\end{align*}
Notice that a since the algorithm is non-preemptive, a job that terminates delays any incomplete job by its full duration. Thus, the event indicates that the `number of times' a job of type 2 delayed any job of type 1 is at least $n^2/2$. Since $\lambda_2>\lambda_1$, this scenario would lead to a large excess cost.

Then, using standard information-theoretic tools, we lower bound the probability that either $E$ occurs in the original scheduling problem or $\bar{E}$ occurs in a problem where the type order has been switched. In both problems, the relevant event causes an excess cost of $\Omega(n^2)$, and substituting the exact probability of this failure case concludes the proof.
\end{proof}


\section{Preemptive Algorithms}\label{sec:preemptive}

In this section, we show how to leverage preemption to get better theoretical and practical performances.

In practice, we allow preemption by discretizing the computation time into small time slots of length $\Delta$. Then, at every iteration, one or multiple job types are selected depending on some algorithm-specific criteria. The current running job(s) of the selected type is allocated computation time $\Delta$ instead of being run to completion. As before, we employ both an explore-then-commit strategy and an optimism-based strategy. In both cases, the only dependence of the resulting algorithm on the discretization size is due to the discretization error (the time between the end of a job and the end of a window), which decreases with the discretization step. We omit that discretization error of at most $\Delta N$ is the bounds.

Note that in practice, any implementation of $\RRR$ proceeds in a similar manner. For instance, in  \cite{motwani1994nonclairvoyant},  the discretization step is assumed much smaller than the length of the jobs. In particular, when we say we run jobs in parallel, in practice, they cyclically run in a $\RRR$ manner with a small discretization step.

\begin{algorithm}[t]
	\caption{Preemptive Algorithms routine}
	\label{algo:preemptive}
\begin{algorithmic}[1]
\STATE \textbf{Init}: type set $\Ucal=[K]$, active jobs $i_k=1,\forall k\in[K]$
\WHILE{$\Ucal$ is not empty}
\STATE Use a type selection subroutine to select a type $k\in\Ucal$
\STATE Run job $i_k$ for $\Delta$ time units
 \IF{$i_k$ was completed}
	\STATE Set $i_k\gets i_k+1$
 \ENDIF
 \IF{All jobs of type $k$ are completed}
	\STATE Remove type $k$ from $\Ucal$
 \ENDIF
 \ENDWHILE
 \end{algorithmic}
\end{algorithm}
\subsection{ETC-RR and UCB-RR}

\paragraph{ETC-RR type selection.} As $\ETCU$, {$\ETCRR$ maintains a set of types $\Acal$ that are candidates for lowest mean size among the set $\Ucal$ of types with at least one remaining jobs. The main difference is that the job type selected is the one in $\Acal$ with the lowest total run-time (not the one with the lowest number of computed jobs).

The statistics needed to construct $\Acal$ are different from the ones used in $\ETCU$.
At a given time, $\beta_{k, \ell}$ is the number of times a job of type $k$ has finished while $\ell$ and $k$ were both active. Moreover, we define
\begin{align*}
\hat{r}_{k, \ell} =  \frac{\beta_{k, \ell}}{\beta_{k, \ell}+\beta_{\ell, k}} \quad \text{and}\quad 
\delta_{k, \ell} = \sqrt{\frac{\log(2n^2K^3)}{2(\beta_{k, \ell}+\beta_{\ell, k})}}.
\end{align*}
The elimination rule is the same as the one of $\ETCU$, using these modified statistics.

 \textit{Reducing the number of algorithm updates}: In practice, both the statistics and the chosen types are not updated at every iteration; active jobs run in parallel (meaning in a round-robin style), and the statistics are updated every time a job terminates. This formulation of the algorithm is the one we implement(see pseudo-code in \Cref{app: preemptive algs}).


 \paragraph{UCB-RR type selection.} 
For each job type $k \in [K]$, we introduce $T_k(t)$, the number of times job type $k$ has been chosen up to iteration $t$, and the random variables $(x_k^s)_{s}$ s.t.:
\begin{align*}
x_k^s = \sum_{t}\ind{a(t) = k, T_k(t)=s \text{ and the job finishes}} .
\end{align*}
It is the indicator that a job of type $k$ is completed when this type is picked for the $s^{th}$ time by the algorithm. We define the empirical means as:
\begin{align*}
\hat{\mu}_k(T):= \frac{1}{T}\sum_{s=1}^{T}x_k^s,
\end{align*}
and define the index for each arm $k$ as
\begin{align*}
u_k(t)= \max \left\{\tilde{\mu} \in[0,1]: d\left(\hat{\mu}_k(T_k(t)), \tilde{\mu}\right) \leq \frac{\log n^2}{T_k(t)}\right\},
\end{align*}
with $d(x,y)$ the Kullback-Leibler divergence between $x$ and $y$. A job type with the largest index is selected.

 \textit{Reducing the number of algorithm updates}: As for $\ETCRR$, the running jobs and statistics are not updated at every iteration. Suppose type $k^*$ is chosen at iteration $t$. If $k^*$ is the last remaining type, it is run until the end. Otherwise, let $\ell$ be the type with the second largest index. We define
 \begin{align*}
\tilde{\mu}^{\gamma}_k(T):= \frac{1}{T+2^{\gamma}}\sum_{s=1}^{T}x_k^s,
\end{align*}
 and 
 \begin{align*}
\tilde{u}^\gamma_k&(t) \\
&\!\!= \max \left\{\tilde{\mu} \in[0,1]: d\left(\tilde{\mu}^{\gamma}_k(T_k(t)), \tilde{\mu}\right) \leq \frac{\log n^2}{T_k(t)+2^\gamma}\right\}.
\end{align*}
This would be the new index of arm $k$, were it to run for additional $2^\gamma$ iterations with no job terminating during this additional iterations.  Then, we set
\[
\gamma^* = \argmax_{\gamma} {\tilde{u}^\gamma_{k^*}(t)\geq u_\ell(t)},
\]
and type $k^*$ is allocated $2^{\gamma^*} $ iterations with no statistics update.

 \subsection{Cost Analysis}
\begin{proposition}
\label{prop:bound-preemptive}
The following bounds hold:
\begin{align*}
	\EE[C_{\ETCRR}]\leq &\EE[C_{\FTPP}]+\frac{12K}{n}\EE[C^{OPT}] \nonumber\\
     & +4n\sqrt{n\log(2n^2K^3)}\sum_{k=1}^{K-1} (K-k)^2\lambda_k.
\end{align*}
and for any $\Delta\leq \frac{\lambda_1}{4}$ and $n\geq \max(20,10\ln(K))$,
\begin{align*}
	\EE[C_{\UCBRR}]&\leq  \EE[C_{\FTPP}]+\frac{12K}{n}\EE[C^{OPT}] \nonumber\\
     & +6n\sqrt{2n\log(2n^2K^2)+2}\sum_{k=1}^{K-1} (K-k)\lambda_k.
	\end{align*}
\end{proposition}
\begin{proof}[Proof sketch (full proof in \Cref{app:preemptive})]
The above proposition is a combination of Propositions \ref{app:bound-ETC-RR} and \ref{prop:bound-UCBRR}.

Both algorithms belong to the following family of type-wise non-preemptive algorithms.
\begin{definition}
Recall that $b_i^{k}$ and $e_i^{k}$ are the beginning and end dates of the computation of the $i^{th}$ job of type $k$. A \textbf{type-wise non-preemptive algorithm} is an algorithm that computes jobs of the same type one after another, i.e., $\forall i \in [n], \forall k \in [K], e_i^{k}\leq b_{i+1}^{k}$. 
\end{definition}
The following Lemma, proven in \Cref{app:cost decomp preemptive} bounds the expected cost of any type-wise non-preemptive algorithms.  
\begin{lemma}[Cost of type-wise non-preemptive algorithms]\label{app:lemma:cost-algo main paper}
  Any type-wise non-preemptive algorithm $A$ has the following upper bound on its cost:
  \begin{align*}
 \EE[C_{A}] \leq &\EE[C_{\FTPP}] \\
  &+  \sum_{\substack{\small(\ell,k) \in [K^2], k>\ell\\(i,j) \in [n]^2}}(\lambda_k-\lambda_\ell)\EE\sbr{\ind{e_j^k< b_i^\ell} }\\
 &+(K-1)n\sum_{k=1}^K\lambda_{k}.
  \end{align*}  
\end{lemma}

\begin{figure}[t]
    \centering
    \includegraphics[width=0.45\textwidth]{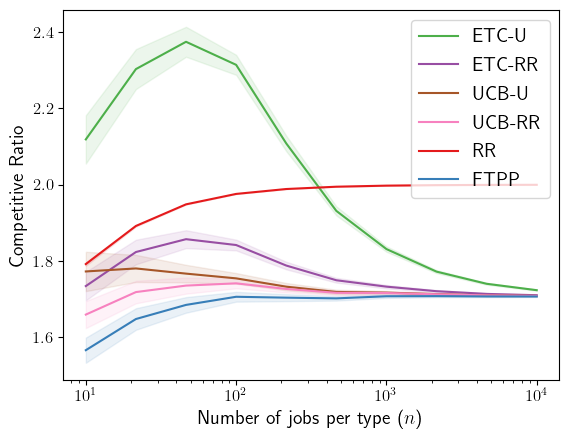} 
    \caption{CR of all algorithms with varying number of jobs, $\lambda_1=1$, $\lambda_2=0.25$, averaged over $400$ seeds.}
    \label{fig:cr linear scale}
\end{figure}

The proof of this Lemma again involves computing explicitly the cost of $\FTPP$ and using that the realization of a job length is independent of its start date. A first upper bound is obtained by noting that a job started before another delays the former in expectation by at most its expected length. The second element to the proof is the fact that at every job termination, at most a job of each other type is currently active. This observation leads to the upper bound on the additional cost of preemption and the last term of the expression of the Lemma. 
Note that this last term implies that our upper bound will include a term scaling as $n\lambda_K$, which would indicate that preemptive algorithms have an extra learning cost scaling as the highest mean type. However, we strongly believe this to be an artefact of the analysis.

Given the decomposition, the two proofs diverge.

The analysis of $\ETCRR$ is split into phases, as the analysis of $\ETCU$. However,  the bound on the number of `bad' jobs computed in each phase requires more care because of independence arguments. Specifically, the upper bound is derived from concentration bounds on the computed statistics, and an additional bound on the number of successful jobs of each type when two types are run in parallel. The details on how to deal with those two non-independent events can be found in \Cref{app:ETC-RR-k-types}.

For $\UCBRR$, the first step is to distinguish two types of `failures' of the index. In the first failure case, the index deviates below the true mean. We show that this happens with probability $O(1/n^2)$ (\Cref{app:boundtaul}), independently of $\Delta$. The second type of failure is when the index of a sub-optimal arm is much larger than its true mean. Here, we show that the upper bound on the number of iterations where this happens does diverge as $\Delta$ goes to zero. However,  the algorithm only incurs a cost on the `bad pull' of a type when the selected job terminates. The probability of job termination decreases as $\Delta$ decreases, which compensates for the rise in the upper bound (\cref{eq:bad_pull_diverge}).

\end{proof}

\section{Experiments}

\begin{figure}[t]
    \centering
    \includegraphics[width=0.45\textwidth]{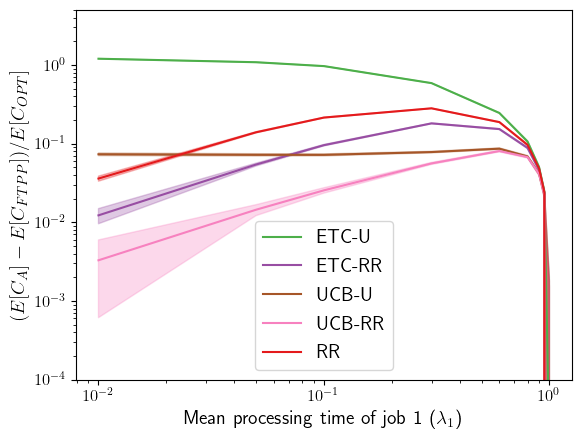} 
    \caption{Normalized excess cost of all algorithms w.r.t. $\FTPP$ with varying value of $\lambda_1$, for $\lambda_2=1$ and $n=50$, averaged over $5,000$ seeds.}
    \label{fig:cr-diff log scale}
\end{figure}

In this section, we design synthetic experiments to compare $\ETCU$, $\ETCRR$, $\UCBRR$,$\UCBU$, $\RRR$ and $\FTPP$. All code is written in Python. We use matplotlib~\cite{hunter2007matplotlib} for plotting, and numpy~\cite{harris2020array} for array manipulations. The above libraries use open-source licenses. Computations are run on a laptop.\footnote{The code to reproduce experiemnts is available at \url{https://github.com/hugorichard/ml4a-scheduling}.} 

The first experiment plots the CR of each algorithm for two types of jobs and fixed values of $\lambda_1,\lambda_2$ as $n$ varies (see \Cref{fig:cr linear scale}). 
Even though all our suggested algorithms have the same asymptotic performance, their non-asymptotic behavior drastically varies. As predicted by theory, the preemptive versions of the algorithm consistently outperform the non-preemptive ones.

In the second experiment, $n=50$ and $\lambda_2=1$ are fixed, while $\lambda_1$ varies in $(0,1)$ (see \Cref{fig:cr-diff log scale}). To be able to discern performance gaps when $\lambda_1$ is small, we plot the difference between the CR of different algorithms and $\FTPP$ at a logarithmic scale. 
Here, for small values of $\lambda$, both preemptive methods outperform the non-preemptive ones. This corresponds with the improvement in the dominant error term of the preemptive cost upper bounds, as a function of $\lambda_2$.

\subsection{Discussion}

\paragraph{Preemptive vs. Non-Preemptive} The competitive ratio of all algorithms is asymptotically the one of $\FTPP$. Indeed, it always holds that $\EE[C_{\OPT}]\geq (\lambda_1+\lambda_2) \frac{n^2}{4}$ (\Cref{eq: optimal cost lower bound}), so by Propositions \ref{prop:bound-preemptive} and \ref{prop:bound-nonpreemptive}, for any algorithm $A$ among $\ETCU$, $\ETCRR$, $\UCBU$ and $\UCBRR$:
\begin{align*}
&\text{CR}_{A} = \text{CR}_{\FTPP}+ \mathcal{O}\left(\sqrt{\frac{\log(n)}{n}}\right).
\end{align*}
On the one hand, the leading term in the cost is the same for all algorithms. On the other hand, the error term can be much smaller in the case of preemptive algorithms.

To illustrate this claim, let us consider the case where there are two types of jobs of expected sizes $\lambda_1$ and $\lambda_2$, respectively. 
Instantiating the bounds of Propositions \ref{prop:bound-nonpreemptive} and \ref{prop:bound-preemptive} to this setting, we get:
\begin{align}
  \EE[C_{\ETCU}] &\leq \EE[C_{\FTPP}] + n  (\lambda_1 + \lambda_2) \sqrt{8n \log(2 n^2 K^3)} \nonumber\\
  &\quad+ \frac{8}{n}\EE[C_{\OPT}], \label{eq:boundETCULambadependency}
\end{align}
and
\begin{align}
	\EE[C_{\ETCRR}]\leq& \EE[C_{\FTPP}] + 2n  \lambda_1(\sqrt{4n \log(2 n^2 K^3)}+1)\nonumber\\
    &+\frac{16}{n}\EE[C_{\OPT}].\label{eq:boundETCRRLambadependency}
\end{align}
If $\lambda_2 \gg \lambda_1$ the bound in Equation \eqref{eq:boundETCULambadependency} is much larger than the bound in Equation \eqref{eq:boundETCRRLambadependency}, which is consistent with what we observe in \Cref{fig:cr-diff log scale}. In particular, one can observe that for small $\lambda_1$, non-preemptive algorithms converge to a strictly positive error (due to the unavoidable dependence in $\lambda_2=1$), while the error of the preemptive algorithms diminishes. This empirically supports our claim that the $n\lambda_K$-dependence, as appears in the preemptive cost decomposition of \Cref{app:lemma:cost-algo main paper}, is only due to a proof artefact.


\paragraph{Optimism-based vs. explore-then-commit. } 

In the simulations, we see that optimism-based algorithms perform much better than their ETC counterparts. In traditional bandit settings, it is well known that the regret of ETC strategies is a constant-times larger than that of optimism-based strategies. Here, we believe that in addition to that,  a  second phenomenon, not reflected in the analysis, renders the optimism-based strategies better than the other ones. Because of the structure of the cost, a pull of a `bad job' at the beginning is much more expensive than the same pull done later in the interaction (as it delays more jobs). Optimism-based strategies explore continuously as they run, whereas ETC strategies have all the exploration at the beginning, when it is more expensive. Again, this phenomenon stands in contrast with traditional bandits, where only the number of `bad pulls' matter, and not their position.

\section{Conclusion and Future Work}

We designed and analyzed a family of algorithms for static scheduling on a single machine in the presence of job types. The special cost structure of this problem differs from that of traditional bandit problems, and early mistakes carry much more weight than late ones, as they delay more jobs. This modified cost directly impacts the performance of algorithms; although all suggested algorithms asymptotically have the same CR as the optimal algorithm that knows job type sizes ($\FTPP$), their non-asymptotic performances differ.

When preemption is allowed, algorithms that explore job types with a strategy inspired by the worst-case optimal deterministic algorithm $\RRR$ have a clear advantage over non-preemptive learning algorithms. Thus, because of the cost structure, the performance is impacted not only by the number of exploratory steps but also by the nature of the exploratory steps. 

Due to the ubiquitousness of scheduling problems, we believe that our results could be extended to many other variants of this setting. In particular, it would be interesting to take our algorithmic principles and test them on real-world scheduling problems. Whether our current assumption on the exponential distribution of job sizes can be removed is an exciting direction for future work. We believe that many of the proofs for the non-preemptive case can be extended to other well-behaved distributions. However, in the preemptive case, our proofs do heavily rely on the properties of the exponential distribution, mainly the memoryless increments property. Without it, we can no longer average over different sections of a job evaluation to estimate its expected duration. Nonetheless, we believe that the results are still generalizable although proving the bounds would be technically much harder.

Moreover, we believe that elements from our works can be taken to other online learning settings outside the scope of scheduling. Specifically, we believe that the notion of types serves as a reasonable approximation that allows the integration of learning to many online problems. We also think that the study of cost functions that are sensitive the early exploration is of great interest.


\section*{Acknowledgements}
\begin{wrapfigure}{r}{0.1\linewidth}
\vspace{-.25cm}
\hspace{-.35cm}
\includegraphics[width=1.2\linewidth]{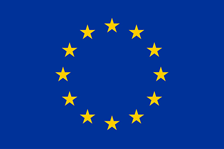}
\vspace{-.2cm}
\end{wrapfigure}
This project has received funding from the European Union’s Horizon 2020 research and innovation programme under the Marie Skłodowska-Curie grant agreement No 101034255. 

Nadav Merlis is partially supported by the Viterbi Fellowship, Technion. 

Mathieu Molina is supported by the French National Research Agency (ANR) through grant ANR-20-CE23-0007.

Vianney Perchet acknowledges support from the French National Research Agency (ANR) under grant number (ANR-19-CE23-0026 as well as the support grant, as well as from the grant “Investissements d’Avenir” (LabEx Ecodec/ANR-11-LABX-0047).
\bibliographystyle{icml2023}
\bibliography{biblio}


\clearpage
\appendix
\onecolumn
\section{Benchmark FTPP}

In this section, for all jobs $i \in [N]$, we call $P_i$ the job size of job $i$. Jobs are ordered in increasing order of their expected size (Notation $P_i$ and $P_{i \mod n}^{\lceil i/n\rceil}$ denote the same job). For any algorithm $A$, we note $T_{ij}^A$ for each $(i,j) \in [N]^2$ the amount of time job $i$ and job $j$ delay each other under algorithm $A$.

\subsection{Cost of OPT and FTPP, CR of $\RRR$}

Let us express the expected cost of any algorithm in terms of $T_{i,j}^A$ for $k \in [K]$:
    \begin{equation}\label{eq:costanyalgo}
		 \EE[C_A]=\mathbb{E}\left[\sum_{i=1}^{N} P_i + \sum_{i=1}^{N} \sum_{j=i+1}^{N} T_{i,j}^A\right] 
   \end{equation}

\begin{lemma}[Cost of OPT]
\label{lemma:cost opt}
The cost of OPT is given by
\[
\EE[C_{\OPT}] = n^2 \left(\sum_{\ell=1}^K \frac{1}{4} \lambda_\ell + \sum_{\ell=1}^K \sum_{k=\ell+1}^K \frac{\lambda_k \lambda_\ell}{\lambda_k + \lambda_\ell}\right) +  \frac{3n}{4} \sum_{\ell=1}^K \lambda_\ell.
\]
\end{lemma}
Note that this lemma implies the following inequality, which will be used in other proofs:
\begin{equation}\label{eq: optimal cost lower bound}
    \EE[C_{\OPT}] \geq \frac{n^2}{4} \left(\sum_{\ell=1}^K  \lambda_\ell\right)
\end{equation}
\begin{proof}
    We apply \Cref{eq:costanyalgo} with $A = \OPT$. In that case, for any to jobs $(i,j)\in [N]$, $i\neq j$, as the shortest job is scheduled first, we have  
    \[
    \EE[T_{ij}^A] = \EE[\min(P_i,P_j)].
    \]
    So $\EE[T_{ij}^A] = \frac{\lambda_k\lambda_\ell}{\lambda_k+\lambda_\ell}$ if job $i$ is of type $k$ and job $j$ is of type $\ell$.
	\begin{align}
		\EE[C_{\OPT}]&=\mathbb{E}\left[\sum_{i=1}^{N} P_i + \sum_{i=1}^{N} \sum_{j=i+1}^{N} T_{i,j}^A\right] \nonumber\\
  &=\sum_{\ell=1}^K\sum_{i=1}^n\left(\lambda_\ell+ \sum_{j=i+1}^n\frac{\lambda_\ell}{2}+\sum_{k=\ell+1}^K\sum_{j=1}^n\frac{\lambda_k\lambda_\ell}{\lambda_k+\lambda_\ell}\right)\nonumber\\
  &=\sum_{\ell=1}^K \left(n \lambda_\ell+\frac{n(n-1)}{2}\frac{\lambda_\ell}{2}+n^2\sum_{k=\ell+1}^K\frac{\lambda_k\lambda_\ell}{\lambda_k+\lambda_\ell}\right)\nonumber\\
  &=n^2 \left(\sum_{\ell=1}^K \frac{1}{4} \lambda_\ell + \sum_{\ell=1}^K \sum_{k=\ell+1}^K \frac{\lambda_k \lambda_\ell}{\lambda_k + \lambda_\ell}\right) +  \frac{3n}{4} \sum_{k=1}^K \lambda_k.\nonumber
  \label{eq:cost_opt}
\end{align}

\end{proof}

\begin{lemma}[Cost of FTPP]
\label{lemma:cost FTPP}
The cost of FTPP is given by:
\[
\EE[C_{\FTPP}] = n^2 \left( \frac12 \sum_{\ell=1}^K \lambda_\ell + \sum_{\ell=1}^K (K - \ell) \lambda_\ell\right) + n\left(\frac{\sum_{\ell=1}^K \lambda_\ell }{2}\right)
\]
\end{lemma}
\begin{proof}
We apply \Cref{eq:costanyalgo} with $A = \OPT$, so that $\EE[T_{ij}^A] = \min(\lambda_k, \lambda_\ell)$ if job $i$ is of type $k$ and job $j$ is of type $\ell$
    \begin{align}
		\EE[C_{\FTPP}]&=\mathbb{E}\left[\sum_{i=1}^{N} P_i + \sum_{i=1}^{N} \sum_{j=i+1}^{N} T_{i,j}^A\right] \nonumber\\
  &=\sum_{\ell=1}^K\sum_{i=1}^n\left(\lambda_\ell+ \sum_{j=i+1}^n\lambda_\ell+\sum_{k=\ell+1}^K\sum_{j=1}^n\lambda_\ell\right)\nonumber\\
  &=\sum_{\ell=1}^K \left(n \lambda_\ell+\frac{n(n-1)}{2}\lambda_\ell+n^2\sum_{k=\ell+1}^K\lambda_\ell\right)\nonumber\\
  &=n^2 \left( \frac12 \sum_{\ell=1}^K \lambda_\ell + \sum_{\ell=1}^K (K - \ell) \lambda_\ell\right) + n\left(\frac{\sum_{\ell=1}^K \lambda_\ell }{2}\right).\nonumber
  \label{eq:cost_opt}
\end{align}
\end{proof}

\begin{lemma}[CR of $\RRR$]
\label{lemma:CR RR}
For any For any $\lambdab$, the following lower bound holds:
\[
\frac{\EE[C_{\RRR}]}{\EE[C_{\OPT}]} \geq 2-\frac{4}{n+3}.
\]
\end{lemma}
\begin{proof}
    For $\RRR$, any two jobs are run in parallel until one terminates, thus:
     \[
    \EE[T_{ij}^{\RRR}] = 2\EE[\min(P_i,P_j)].
    \]
    Thus, by equation \ref{eq:costanyalgo}: 
    \[
    \EE[C_{\RRR}] = \sum_{i=1}^N\EE[P_i]+\sum_{j=1}^N 2\EE[\min(P_i,P_j)].
    \]
    On the other hand:
    \[
    \EE[C_{\OPT}] = \sum_{i=1}^N\EE[P_i]+\sum_{j=1}^N \EE[\min(P_i,P_j)].
    \]
    Thus:
    \begin{align*}
      \frac{\EE[C_{\RRR}]}{\EE[C_{\OPT}]}&= 2-\frac{\sum_{i=1}^N P_i}{\EE[C_{\OPT}]}\\
      &=2-\frac{n\sum_{\ell=1}^K \lambda_\ell}{n^2 \left(\sum_{\ell=1}^K \frac{1}{4} \lambda_\ell + \sum_{\ell=1}^K \sum_{k=\ell+1}^K \frac{\lambda_k \lambda_\ell}{\lambda_k + \lambda_\ell}\right) +  \frac{3n}{4} \sum_{\ell=1}^K \lambda_\ell}\\
      &\geq 2-\frac{n\sum_{\ell=1}^K \lambda_\ell}{n^2 \left(\sum_{\ell=1}^K \frac{1}{4} \lambda_\ell \right) +  \frac{3n}{4} \sum_{\ell=1}^K \lambda_\ell}\\
      &=2-\frac{4}{n+3}.
    \end{align*}
With the second line obtained by Lemma \ref{lemma:cost opt}.
\end{proof}
\subsection{CR of FTPP}

\subsubsection{CR with $K$ types}
\label{app:CR-FTPP-Ktypes}

\begin{proposition}[Upper bound on the CR in function of $\lambdab$]
\label{prop:CR-FTPP-Ktypes}
The CR of FTPP with $K$ types of jobs with $n$ jobs per type satisfies:
\[
\frac{\EE[C_{\FTPP}]}{\EE[C_{\OPT}]} \leq 2 - f_K(\lambdab)
\]
where $f_K(\lambdab) = \frac{ 2 \sum_{\ell=1}^K \sum_{k=\ell+1}^K \frac{\lambda_k \lambda_\ell}{\lambda_k + \lambda_\ell}- \sum_{\ell=1}^K (K - \ell) \lambda_\ell}{\sum_{\ell=1}^K \frac{1}{4} \lambda_\ell + \sum_{\ell=1}^K \sum_{k=\ell+1}^K \frac{\lambda_k \lambda_\ell}{\lambda_k + \lambda_\ell}}$
\end{proposition}

Note that instantiating this bound with $K=2$ types of jobs, $n$ jobs per type, $\lambda_1 = 1$ and $\lambda_2 = \lambda > 1$, we get  \Cref{prop:CR-FTPP-2types}:
\[
\frac{\EE[C_{\FTPP}]}{\EE[C_{\OPT}]} \leq 2 - 4\frac{\lambda-1}{{(1 + \lambda)^2 + 4 \lambda}}.
\]
\begin{proof}[Proof of \Cref{prop:CR-FTPP-Ktypes}]

    Compute $\EE[C_{\OPT}]$ using \Cref{lemma:cost opt}, $\EE[C_{\FTPP}]$ using \Cref{lemma:cost FTPP}.
    
	The competitive ratio of $\FTPP$ is given by:
    \begin{equation}
      \label{eq:cr:ftpp}
	\CR_{\FTPP} = \frac{\EE[C_{\FTPP}]}{\EE[C_{\OPT}]} =\frac{n^2 \left( \frac12 \sum_{\ell=1}^K \lambda_\ell + \sum_{\ell=1}^K (K - \ell) \lambda_\ell\right) + n(\frac{1}{2}\sum_{\ell=1}^K \lambda_\ell )}{n^2 \left(\sum_{\ell=1}^K \frac{1}{4} \lambda_\ell + \sum_{\ell=1}^K \sum_{k=\ell+1}^K \frac{\lambda_k \lambda_\ell}{\lambda_k + \lambda_\ell}\right) + n( \frac{3}{4} \sum_{\ell=1}^K \lambda_\ell)}
    \end{equation}

For any values $a,b,c,d \in \RR_+^4$, 
\begin{equation}\label{eq:boundfrac}
  \text{if $a>c>0$ and $d>b>0$, then } \frac{a+b}{c+d}\leq \frac{a}{c}. 
\end{equation}

Now, we have $\frac{1}{2}\leq \frac{3}{4}$ and $$\sum_{\ell=1}^K \frac{1}{4} \lambda_\ell + \sum_{\ell=1}^K \sum_{k=\ell+1}^K \frac{\lambda_k \lambda_\ell}{\lambda_k + \lambda_\ell}\leq \frac12 \sum_{\ell=1}^K \lambda_\ell + \sum_{\ell=1}^K (K - \ell) \lambda_\ell.$$
This implies:

    \begin{align*}
     \CR_{\FTPP}&\leq \frac{\frac12 \sum_{\ell=1}^K \lambda_\ell + \sum_{\ell=1}^K (K - \ell) \lambda_\ell}{\sum_{\ell=1}^K \frac{1}{4} \lambda_\ell + \sum_{\ell=1}^K \sum_{k=\ell+1}^K \frac{\lambda_k \lambda_\ell}{\lambda_k + \lambda_\ell}}\\
     &= 2 - \underbrace{\frac{ 2 \sum_{\ell=1}^K \sum_{k=\ell+1}^K \frac{\lambda_k \lambda_\ell}{\lambda_k + \lambda_\ell}- \sum_{\ell=1}^K (K - \ell) \lambda_\ell }{\sum_{\ell=1}^K \frac{1}{4} \lambda_\ell + \sum_{k=1}^K \sum_{k=\ell+1}^K \frac{\lambda_k \lambda_\ell}{\lambda_k + \lambda_\ell}}}_{f_K(\lambdab)}.
    \end{align*}
\end{proof}

\subsubsection{Upper bound on the CR for particular values of $\lambdab$}
\label{sec:CR-FTTP-Ktypes3}
\begin{proposition}
\label{prop:CR-FTPP-Ktypes3}
\[
\forall K > 1, \exists \lambdab, 0 < \lambda_1 \leq \cdots \leq \lambda_K=1, \CR_{\FTPP}(\lambdab) \leq \frac{H_K-\frac{1}{2}B_K}{\frac{1}{4}B_K+A_K}
\]
with $H_K=\sum_{k=1}^K \frac{1}{k}$, $B_K=\sum_{k=1}^K \frac{1}{k^2}$ and $A_K=\sum_{k=1}^K \sum_{\ell=1}^{k-1} \frac{1}{k^2+\ell^2}$.
\end{proposition}

Furthermore $\lim_{K \rightarrow \infty} \frac{H_K-\frac{1}{2}B_K}{\frac{1}{4}B_K+A_K} = \frac{4}{\pi} \approx 1.273$ which implies that there exists some value of $K$ for which $\CR_{\FTPP}(\lambdab) \leq 1.274$.

\begin{proof}
A way to prove such a result would be to find the minimum of $\CR_{\FTPP}$ with respect to $\lambda$. But this is difficult. We propose another point $\tilde{\boldsymbol{\lambda}}$.
\begin{equation}
	\tilde{\lambda}_k=\frac{1}{(K-k+1)^2}.
\end{equation}

We express the competitive ratio using $\tilde{\boldsymbol{\lambda}}$:
\begin{align*}
\CR_{\FTPP}(\tilde{\boldsymbol{\lambda}})&= \frac{\sum_{k=1}^K (\frac{1}{2}+K-k)\tilde{\lambda}_k}{\sum_{k=1}^K \left(\frac{1}{4} \tilde{\lambda}_k +\sum_{\ell=k+1}^K \frac{\tilde{\lambda}_k \tilde{\lambda}_\ell}{\tilde{\lambda}_k+ \tilde{\lambda}_\ell} \right)}\\
&= \frac{\sum_{k=1}^K (\frac{1}{2}+K-k)\frac{1}{(K-k+1)^2}}{\sum_{k=1}^K \left(\frac{1}{4} \frac{1}{(K-k+1)^2} +\sum_{\ell=k+1}^K \frac{1}{(K-k+1)^2+(K-\ell+1)^2} \right)} \\
&=\frac{\sum_{k=1}^K (\frac{1}{2}+k-1)\frac{1}{k^2}}{\sum_{k=1}^K \left( \frac{1}{4} \frac{1}{k^2} +\sum_{\ell=1}^{k-1} \frac{1}{k^2+\ell^2}\right)}\\
&=\frac{H_K-\frac{1}{2}B_K}{\frac{1}{4}B_K+A_K} \\
\text{with} \quad &H_K=\sum_{k=1}^K \frac{1}{k}, \quad B_K=\sum_{k=1}^K \frac{1}{k^2}, \quad \text{and} \quad A_K=\sum_{k=1}^K \sum_{\ell=1}^{k-1} \frac{1}{k^2+\ell^2}.
\end{align*}
This shows the first part of the lemma. The fact that $\frac{H_K-\frac{1}{2}B_K}{\frac{1}{4}B_K+A_K} \rightarrow \frac{4}{\pi}$ follows from \Cref{lemma:limit}.
\end{proof}

\begin{lemma}
\label{lemma:limit}
\[
\lim_{K \rightarrow \infty} \frac{H_K-\frac{1}{2}B_K}{\frac{1}{4}B_K+A_K} = \frac{4}{\pi}
\]
with $H_K=\sum_{k=1}^K \frac{1}{k}$, $B_K=\sum_{k=1}^K \frac{1}{k^2}$ and $A_K=\sum_{k=1}^K \sum_{\ell=1}^{k-1} \frac{1}{k^2+\ell^2}$
\end{lemma}
\begin{proof}

We now focus on the behavior of $\frac{H_K-\frac{1}{2}B_K}{\frac{1}{4}B_K+A_K}$ as $K$ goes to $\infty$.

We know that for the harmonic number $H_K=\Theta(\log(K))$, and that for the partial sum of the Basel problem $0\leq B_K \leq \sum_{k=1}^{\infty} k^{-2}=\pi^2/6=\mathcal{O}(1)$.
Let us bound $A_k$. Using the fact that for $y>0$ and $x>0$ the function $f:(x,y) \mapsto (x^2+y^2)^{-1}$ is decreasing in $x$, for $(k,\ell)\in [K]^2$ we have
\begin{align*}
\int_\ell^{\ell+1} \frac{1}{k^2+t^2}dt \leq &\frac{1}{k^2+\ell^2}\leq \int_{\ell-1}^{\ell} \frac{1}{k^2+t^2} \\
\frac{1}{k^2}\int_\ell^{\ell+1} \frac{1}{(t/k)^2+1}dt\leq &\frac{1}{k^2+\ell^2}\leq \frac{1}{k^2}\int_{\ell-1}^{\ell} \frac{1}{(t/k)^2+1}dt\\
\frac{1}{k}(\arctan(\frac{\ell+1}{k})-\arctan(\frac{\ell}{k})) \leq &\frac{1}{k^2+\ell^2}\leq \frac{1}{k}(\arctan(\frac{\ell}{k})-\arctan(\frac{\ell-1}{k})).
\end{align*}
Hence by summing for $1 \leq \ell < k \leq K$:
\begin{align*}
\sum_{k=1}^K \frac{1}{k}(\arctan(1) - \arctan(\frac{1}{k})) &\leq A_K \leq \sum_{k=1}^K \frac{1}{k} (\arctan(\frac{k-1}{k}) -\arctan(0)),\\
\sum_{k=1}^K \frac{1}{k}(\frac{\pi}{4} - \arctan(\frac{1}{k})) &\leq A_k \leq \sum_{k=1}^K \frac{1}{k} \arctan(\frac{k-1}{k}).
\end{align*}
For the right-hand-side we use that $\arctan$ is increasing, thus
\begin{equation*}
A_K\leq  \sum_{k=1}^K \frac{1}{k} \arctan(\frac{k-1}{k}) \leq \frac{\pi}{4} H_K.
\end{equation*}
Using that $\arctan(x)\leq x$ for $x\geq 0$, we have 
\begin{equation*}
A_K \geq \sum_{k=1}^K \frac{1}{k}(\frac{\pi}{4} - \frac{1}{k})=\frac{\pi}{4}H_K-B_K.
\end{equation*}
Combining everything we obtain the following inequality:
\begin{equation*}
\frac{H_K-\frac{1}{2}B_K}{\frac{\pi}{4}H_K+\frac{1}{4}B_K} \leq \CR_{\FTPP}(\tilde{\boldsymbol{\lambda}}) \leq \frac{H_K-\frac{1}{2}B_K}{\frac{\pi}{4}H_K-\frac{3}{4}B_K}.
\end{equation*}
Therefore
\begin{equation*}
\lim_{K\rightarrow \infty} \CR_{\FTPP}(\tilde{\boldsymbol{\lambda}}) = \frac{4}{\pi} \approx 1.273.
\end{equation*}
\end{proof}

\subsubsection{The cost of FTPP is lower than the cost of RR}\label{app:FTPPvsRR}
Let us order all jobs $i \in [N]$ in order of their increasing expected size, and denote $P_i$, the size of job $i$. An alternative notation to $P_i$ is $P_{i \mod n}^{\lceil i/n\rceil}$, where the first is used in this proof for convenience. 
We consider here the most general setting where $K=N$.

We have
\begin{lemma}
\label{lemma:ftpp leq rr}
\[
\EE[C_{\FTPP}] \leq \EE[C_{\RRR}] 
\]
\end{lemma}

\begin{proof}
The cost of $\FTPP$ with $K = N$ and $n=1$ is given by
\begin{align*}
    \EE[C_{\FTPP}] &= \sum_{i=1}^N \EE[P_i] + \sum_{i=1}^N \sum_{j=i + 1}^N \EE\left[T_{ij}^{\FTPP}\right] \\
    &= \sum_{i=1}^N \EE[P_i] + \sum_{i=1}^N \sum_{j=i + 1}^N \min(\lambda_i, \lambda_j) \\
\end{align*}
where $T_{ij}^{\FTPP}$ is the amount of time job $i$ and job $j$ delay each other in $\FTPP$ which verifies 

$\EE\left[T_{ij}^{\FTPP}\right] = \min(\lambda_i, \lambda_j)$

Similarly, using $T_{ij}^{\RRR} = 2\min(P_i, P_j)$ which implies $\EE[T_{ij}^{\RRR}] = 2 \frac{\lambda_i \lambda_j}{\lambda_i + \lambda_j}$, we get
\begin{align*}
    \EE[C_{\RRR}] &= \sum_{i=1}^N \EE[P_i] + \sum_{i=1}^N \sum_{j=i + 1}^N \EE\left[T_{ij}^{\RRR}\right] \\
    &= \sum_{i=1}^N \EE[P_i] +  \sum_{i=1}^N \sum_{j=i + 1}^N 2 \frac{\lambda_i \lambda_j}{\lambda_i + \lambda_j}\\
\end{align*}

Then we write 
\begin{align*}
   2 \frac{\lambda_i \lambda_j}{\lambda_i + \lambda_j} &= \frac{2}{\frac1{\lambda_i} + \frac1{\lambda_j}}\\
   &\geq \frac{2}{\frac1{\min(\lambda_i, \lambda_j)} + \frac1{\min(\lambda_i, \lambda_j)}} \\
   &\geq \min(\lambda_i, \lambda_j)
\end{align*}

We conclude that $C_{\FTPP} \leq C_{\RRR}$.
\end{proof}

\subsection{Lower bound: Proof of Proposition~\ref{prop:lowerbound}}
\label{app:lowerbound}
Let us order all jobs $i \in [N]$ in order of their increasing expected size, and denote $P_i$, the size of job $i$. An alternative notation to $P_i$ is $P_{i \mod n}^{\lceil i/n\rceil}$, where the first is used in this proof for convenience. 
We consider here the most general setting where $K=N$.
Any algorithm has a cost:
\begin{align*}
		\EE[C^{A}] & = \sum_{i=1}^N \EE[P_i] + \sum_{i=1}^N \sum_{j=i+1}^N \EE[T^{A}_{ij}]                                  
\end{align*}
where $T^A_{ij} = D^A_{ij} + D^A_{ji}$ where $D^A_{ij}$ is the amount of time job $i$ delay job $j$.

\begin{lemma}\label{app:lowerbound:lemma2}
Consider $K=N$ jobs where job $i \in [N]$ has mean size $\lambda_i$ and $\lambda_1 \leq \dots \leq \lambda_N$.
Consider any algorithm $A$ and let $T^A_{ij}$ the total amount of time spent by
$A$ on $i$ or $j$ while both jobs are alive.
\[
\EE[T^{A}_{ij}] \geq 2 \EE[T^{OPT}_{ij}]
\]
where $\OPT$ is the optimal offline algorithm
\end{lemma}
\begin{proof}[Proof of Lemma~\ref{app:lowerbound:lemma2}]
	Let us first prove our proposition for any deterministic algorithm $A$.
	We denote $i(t)$ amount of time that $A$ allocates to job $i$ after a time $t < T_{ij}^A$ is allocated to job $i$ or $j$.

	\begin{align*}
		 & \EE[T^A_{ij}]    \\                       &  = \int_{t=0}^{+ \infty} \PP(T^A_{ij} \geq t) dt                                                                                        \\
		 & = \int_{t=0}^{+ \infty} \PP\left(P_i \geq i(t)\right) \PP\left(P_j \geq t-i(t)\right) dt                                                                                                     \\
		 & = \int_{t=0}^{+ \infty} \exp\left(-\frac{i(t)}{\lambda_i}\right)  \exp\left(-\frac{t - i(t)}{\lambda_j}\right)dt                                                                  \\
		 & = \int_{t=0}^{+ \infty} \exp\left(-\frac{i(t)+t/2-t/2}{\lambda_i}\right)  \exp\left(-\frac{t - \left(i(t)+t/2-t/2\right)}{\lambda_j}\right)dt                                                                  \\
		 & = \int_{t=0}^{+ \infty} \exp\left(-\left(\frac{1}{\lambda_i} + \frac{1}{\lambda_j}\right) \frac{t}{2}\right) \exp\left(-\left(\frac{1}{\lambda_i} - \frac{1}{\lambda_j}\right)\left(i(t) - \frac{t}{2}\right)\right)dt.
	\end{align*}

	Calling $f(t) = \exp\left(-\left(\frac{1}{\lambda_i} + \frac{1}{\lambda_j}\right) \frac{t}{2}\right)$ and $g(t) = |\frac{1}{\lambda_i} - \frac{1}{\lambda_j}|\left(i(t) - \frac{t}{2}\right)$ it holds that either
	\[\int_{t=0}^{\infty} f(t) \exp(-g(t))dt \geq \int_{t=0}^{\infty} f(t)dt\]
	or
	\[\int_{t=0}^{\infty} f(t) \exp(g(t))dt \geq \int_{t=0}^{\infty} f(t)dt. \]
	Otherwise, we would have
	\[
		\int_{t=0}^{\infty} f(t) \frac{1}{2}(\exp(-g(t)) + \exp(g(t)))dt < \int_{t=0}^{\infty} f(t) dt
	\]
	which cannot be true since $\forall t, \frac{1}{2}(\exp(-t) + \exp(t)) \geq 1$.

	Therefore an adversary knowing $i(t)$ can always chose the order of $\lambda_i$ and $\lambda_j$ such that
	\[
		\EE[T^A_{ij}] \geq \int_{t=0}^{+ \infty} \exp(-(\frac{1}{\lambda_i} + \frac{1}{\lambda_j}) \frac{t}{2})dt = 2 \frac{\lambda_i \lambda_j}{\lambda_i + \lambda_j}
	\]

	The optimal delay is
	\[
		\EE[T^{OPT}_{ij}] = \EE[\min(P_i, P_j)] = \frac{\lambda_i \lambda_j}{\lambda_i + \lambda_j}
	\]
	
	so our Lemma is proven for any deterministic algorithm $A$.

	Consider a randomized algorithm $R$ which can be seen as a probabilistic distribution over the set of deterministic algorithms. Therefore $A$, $i(t)$ and $g(t)$ are now seen as random variables.
	By the tower rule, the amount of time job $i$ and $j$ delay each other in $R$ is such that:
	\begin{align*}
		\EE[T_{ij}^R] & = \EE[\EE[T_{ij}^A | A]]                                                             \\
		              & = \EE[\int_{t=0}^{+ \infty} f(t) \exp(\textrm{sign}(\lambda_i - \lambda_j) g(t)) dt]
	\end{align*}

	By the same argument as in the deterministic case, it holds that either
	\[\EE[\int_{t=0}^{\infty} f(t) \exp(-g(t))dt] \geq \int_{t=0}^{\infty} f(t)dt\]
	or
	\[\EE[\int_{t=0}^{\infty} f(t) \exp(g(t))dt] \geq \int_{t=0}^{\infty} f(t)dt \]
	Otherwise, we would have
	\[
		\EE[\int_{t=0}^{\infty} f(t) \frac{1}{2}(\exp(-g(t)) + \exp(g(t)))dt] < \int_{t=0}^{\infty} f(t) dt
	\]
	which implies that there exists a deterministic function $g$ such that
	\[
		\int_{t=0}^{\infty} f(t) \frac{1}{2}(\exp(-g(t)) + \exp(g(t)))dt < \int_{t=0}^{\infty} f(t) dt
	\]
	which cannot be true as shown in the deterministic case.
	The rest of the argument is the same as in the deterministic case and therefore omitted.
\end{proof}

Now we are ready to prove Proposition~\ref{prop:lowerbound}.

\begin{proof}[Proof of Proposition~\ref{prop:lowerbound}]
    Take any algorithm $A$
    \begin{align}
		\EE[C_A] & = \sum_{i=1}^N \EE[P_i] + \sum_{i=1}^N \sum_{j=i+1}^N \EE[T^A_{ij}]                                           \\
		        & \geq \sum_{i=1}^N \lambda_i + 2 \sum_{i=1}^N \sum_{j=i+1}^N \EE[T^{OPT}_{ij}] \label{applbeq1}
	\end{align}
	where~\eqref{applbeq1} comes from Lemma~\ref{app:lowerbound:lemma2}.

 Observe that applying $\RRR$ on the same data would yield an expected completion time:
 \begin{align*}
 \EE[C_{\RRR}] &= \sum_{i=1}^N \EE[P_i] + 2 \sum_{i=1}^N \sum_{j=i+1}^N \EE[\min(P_i, P_j)] \\
 &= \sum_{i=1}^N \EE[P_i] + 2 \sum_{i=1}^N \sum_{j=i+1}^N \EE[T^{OPT}_{ij}] \\
 &\leq \EE[C_A]
 \end{align*}
which concludes the proof.
\end{proof}

\clearpage

\section{Analysis of Non-Preemptive Learning algorithms}\label{app:non-preemptive}

\subsection{Full Algorithmic Details}
\label{app: non-preemptive algs}
In this appendix, we present a full description of $\ETCU$ and $\UCBU$.

\begin{algorithm}[H]
	\caption{Explore-Then-Commit Uniform ($\ETCU$)]}
	\label{algo:etcu}
	\begin{algorithmic}[1]
		\STATE {\bfseries Input :} $n \geq 1$ (number of jobs of each type), $K \geq 2$ (number of types)
		\STATE For all pairs of different types $k, \ell$ initialize $\delta_{k, \ell} = 0$, $\hat{r}_{k, \ell} = 0$ and $h_{k, \ell} = 0$
		\STATE For all types $k$, set $m_{k} = 0$ 
		\REPEAT
        \STATE $\mathcal{U}$ is the set of types with at least one remaining job
        \IF{$\mathcal{A}$ is empty}
        \STATE $\mathcal{A} = \{ \ell \in \mathcal{U}, \forall k \in \mathcal{U}, k \neq \ell , \enspace \hat{r}_{k, \ell} - \delta_{k, \ell} \leq 0.5 \} $ 
        \ENDIF
        \STATE Select the type $\ell$ with the lowest number of finished jobs $\ell = \argmin_{k \in \mathcal{A}} m_{k}$ and run one job of type $\ell$ yielding a size $P_{m_{\ell} + 1}^{\ell}$.
        \STATE $m_{\ell} = m_{\ell} + 1$
		\FOR {$k, \ell$ in $\mathcal{A}$, $k \neq \ell$}
		\STATE $h_{k, \ell} = \sum_{i=1}^{\min(m_{k}, m_{\ell})} \mathds{1}\{P^k_i < P^{\ell}_i\}$
		\STATE $\delta_{k,\ell} = \sqrt{\frac{\log(2n^2K^4)}{2 \min(m_{k}, m_{\ell})}}$
		\STATE $\hat{r}_{k, \ell} = \frac{h_{k, \ell}}{\min(m_{k}, m_{\ell})}$
        \IF{$\hat{r}_{k, \ell} - \delta_{k, \ell} \geq 0.5$ or $m_\ell=n$} 
        \STATE Remove $\ell$ from $\mathcal{A}$
        \ENDIF
		\ENDFOR
		\UNTIL{$ \mathcal{U}$ is not empty} 
	\end{algorithmic}
\end{algorithm}

\begin{algorithm}[H]
	\caption{Upper-Confidence-Bound-Uniform ($\UCBU$)}
	\label{algo:ucbu}
	\begin{algorithmic}[1]
		\STATE {\bfseries Input :} $n \geq 1$ (number of jobs of each type), $K \geq 2$ (number of types)
		\STATE For all types $k \in [K]$, set $m_{k} = 0$
        \STATE Set $\Ucal = [K]$
        \STATE  For all types $k \in [K]$, compute the lower bound $\underline{\lambda}_k^{m_k}$ using \Cref{app:ucbu:lb}
		\REPEAT
        \STATE Select $k^* = \argmin_{k \in \Ucal} \underline{\lambda}_k^{m_k}$
        \STATE Set $m_{k^*} = m_{k^*} + 1$
        \STATE Compute a job of type $k^*$ until completion and record its size $P_{k^*}^{m_{k^*}}$
        \STATE Update the lower bound $\underline{\lambda}_{k^*}^{m_{k^*}}$ using again \Cref{app:ucbu:lb}
        \STATE If $m_{k^*} = n$, remove $k^*$ from $\Ucal$
		\UNTIL{$ \Ucal$ is empty}

 \end{algorithmic}
\end{algorithm}

\subsection{Cost Decomposition}\label{app:cost decomp non-preemptive}
In this appendix, we analyze the non-preemptive learning algorithms presented in our paper - $\ETCU$ and $\UCBU$. We start by presenting a general cost decomposition result that relates the cost of any non-preemptive algorithm to the one of $\FTPP$. We will use this result to derive the bounds of both our suggested algorithms.
\begin{lemma}[Cost of non-preemptive algorithms]
  \label{app:lemma:cost-non-preemptive}
  Any non-preemptive algorithm $A$ has a cost
  \[
\EE[C_{A}] = \EE[C_{\FTPP}] + \sum_{\ell=1}^K \sum_{k=\ell+1}^K\sum_{(i,j) \in [n]^2}( \lambda_k-\lambda_\ell)\EE\sbr{\ind{P_i^k \text{ computed before } P_j^\ell}}
  \]
\end{lemma}
\begin{proof}
  Denote $P_i$, the size of the job $i$, and $T^A_{ij} = D^A_{ij} + D^A_{ji}$, where $D^A_{ij}$ is the amount of time a job $i$ delays job $j$.
  For any algorithm, we have:
  \[
C_{A} = \sum_{a=1}^N P_a + \sum_{a=1}^N \sum_{b=a+1}^N T^A_{ab}
  \]
  For non-preemptive algorithms, $T^A_{ab} = P_a$ if job $a$ is scheduled before $b$ and $P_b$ otherwise so that we can write

  \[
C_{A} = \sum_{a=1}^N P_a + \sum_{a=1}^N \sum_{b=a+1}^N  \rbr{P_{a}\ind{P_a \text{ computed before } P_b} + P_b\ind{P_b \text{ computed before } P_a}}
  \]

  Now assume w.l.o.g. that $(P_a)_{a\in\sbr{N}}$ are in the order chosen by $\FTPP$, i.e., $P_a$ is the $a^{th}$ executed task by $\FTPP$ and if $a \le b$ then $\EE[P_a] \leq \EE[P_b]$. Under this convention, we get:
  \begin{align*}
    C_{\FTPP} = \sum_{a=1}^N P_a + \sum_{a=1}^N \sum_{b=a+1}^N  P_{a}
  \end{align*}
  and recalling that 
  $$\ind{P_a \text{ computed before } P_b} = 1-\ind{P_b \text{ computed before } P_a}$$
  we have
  \[
C_{A} = C_{\FTPP} + \sum_{a=1}^N \sum_{b=a+1}^N (P_b - P_a) \ind{P_b \text{ computed before } P_a}
  \]

  Reindexing the job without changing the order, where $P_i^k$ is now the $i$-th job of type $k$, we have:
  \[
    C_{A} = C_{\FTPP} + \sum_{\ell=1}^K \sum_{j=1}^n \sum_{k=\ell+ 1}^K   \sum_{i=1}^n (P_i^k - P_j^\ell) \ind{P_i^k \text{ computed before } P_j^\ell}
  \]
  Taking the expectation finishes the proof.
\end{proof}

\subsection{Upper bound for ETC-U}\label{app:bound_ETCUk}

\begin{proposition}\label{app:fullboundETCU}
The following upper bounds hold:
\begin{align*}
        \EE[C_{\ETCU}] 
    \leq \EE[C_{\FTPP}] +\frac{1}{n}\EE[C_{\OPT}]+\sum_{k\in [K]}\left[\frac{1}{2}(k-1)(2K-k)+(K-k)^2\right]\lambda_k n \sqrt{8n \log(2 n^2K^3)}.
\end{align*}
and 
\begin{align*}
        \EE[C_{\ETCU}] 
    \leq \EE[C_{\FTPP}] +\frac{1}{n}\EE[C_{\OPT}]+\sum_{k\in [K]}\sum_{\ell=1}^{k-1}(K-\ell)\frac{\left(\lambda_k+\lambda_\ell\right)^2 }{(\lambda_k-\lambda_\ell)}n 8 \log(2 n^2K^3).
\end{align*}
\end{proposition}
We start with the following technical lemma, isolated to be reused in other proofs. Pick some $\alpha\in \mathbb{N}$. Let $(X_i^{1})_{i \in [\alpha n]}$ and $(X_i^{2})_{i \in [\alpha n]}$ be independent exponential variables of parameters $\lambda_1$ and $\lambda_2$ respectively. Define for any $m\in [\alpha n]$:
\[
\hat{r}^m = \frac{1}{m}\sum_{i=1}^m \mathds{1}_{X_i^{1}<X_i^{2}}
\]
and
\[
\delta^{(m,n)}= \sqrt{\frac{\log(2n^2K^3)}{2m}}.
\]
Let $r$ denote the expectation $ r:= \EE\left[\mathds{1}_{X_i^{1}<X_i^{2}}\right]=\frac{\lambda_2}{\lambda_1+\lambda_2}$. 
\begin{lemma}\label{lem:concentration_rm}
For any $m\in [\alpha n]$, the estimator $\hat{r}^m$ is within $\delta^{(m,n)}$ of its expectation w.h.p:
\[
 \PP\left( \exists m \in [\alpha n] \text{ s.t. }|\hat{r}^m - r| \geq \delta^{(m, n)} \right) \leq \frac{\alpha}{nK^3}.
\]
\end{lemma}
\begin{proof}
By Hoeffding's inequality:
\begin{align*}
    \forall m \in [\alpha n], & \enspace\PP\left( |\hat{r}^m - r| \geq \sqrt{\frac{\log(2 n^2K^3)}{2m}}\right) \leq \frac{1}{n^2K^3}\\
\end{align*}
The lemma is then obtained by a union bound over the $\alpha n$ possible values of $m$.
\end{proof}

We are now ready to prove Proposition \ref{app:fullboundETCU}.
\begin{proof}
Recall that we assumed without loss of generality that $\lambda_1 \leq \cdots \leq \lambda_K$. Recall also the definition for any $(k,\ell) \in [K]^2$, for any $(m_\ell,m_k) \in [n]^2$, of:
\begin{align*}
    \hat{r}_{k, \ell}^{\min(m_{k}, m_{\ell})} = \frac{1}{\min(m_{k}, m_{\ell})} \sum_{i=1}^{\min(m_{k}, m_{\ell})}\mathds{1}_{P_i^k<P_i^\ell}.
\end{align*}
Let us define the good event $\Ecal$ as:
\[
\Ecal:= \left\{\forall (k,\ell) \in [K]^2, \forall m \in [n], |\hat{r}^m_{k,\ell}-\EE[r^m_{k,\ell}]|<\delta^{(m,n)}\right\}
\]

By Lemma \ref{lem:concentration_rm} applied with $\alpha=1$, for any couple $(\ell, k)$ it holds that :
\[
\PP\left(\exists m \in [n] \text{ s.t. } |\hat{r}^m_{k,\ell}-\EE[r^m_{k,\ell}]|>\delta^{(m,n)}\right)\leq \frac{1}{n}.
\]
A union bound over the $\frac{K(K-1)}{2}$ possible pairs gives the following bound:
\begin{equation}\label{eq:concentration_good_event_ETCU}
    \PP\left(\overline{\Ecal}\right)\leq \frac{1}{2nK}.
\end{equation}

With the help of \Cref{app:lemma:cost-non-preemptive}, the cost of $\ETCU$ can be decomposed using the event $\Ecal$ as follows:
\begin{align}
    \EE[C_{\ETCU}] 
    & = \EE[C_{\FTPP}] + \sum_{\ell=1}^K \sum_{k=\ell+1}^K\sum_{(i,j) \in [n]^2}( \lambda_k-\lambda_\ell)\EE\sbr{\ind{P_i^k \text{ computed before } P_j^\ell}} \tag{\Cref{app:lemma:cost-non-preemptive}}\\
    &\le\EE[C_{\FTPP}] +  \underbrace{\sum_{\ell=1}^K \sum_{k=\ell+1}^K\sum_{(i,j) \in [n]^2}( \lambda_k-\lambda_\ell)\EE\sbr{\ind{P_i^k \text{ computed before } P_j^\ell}\vert \Ecal} }_{(i)} \nonumber\\
    & \qquad\qquad\quad\;\, + \underbrace{\sum_{\ell=1}^K \sum_{k=\ell+1}^K\sum_{(i,j) \in [n]^2}( \lambda_k-\lambda_\ell)\EE\sbr{\ind{P_i^k \text{ computed before } P_j^\ell}\vert \overline{\Ecal}}\PP\rbr{\overline{\Ecal}}}_{(ii)}. \label{eq:etcu decomposition}
\end{align}
\paragraph{Bounding $(ii)$.} 

Recall that by assumption, if $k\ge\ell$, then $\lambda_k\ge\lambda_\ell$. Therefore, we have that
\begin{align}
(ii) &= \PP\rbr{\overline{\Ecal}}\sum_{\ell=1}^K \sum_{k=\ell+1}^K\sum_{(i,j) \in [n]^2}\underbrace{(\lambda_k-\lambda_\ell)}_{\ge0}\EE\sbr{\underbrace{\ind{P_i^k \text{ computed before } P_j^\ell} }_{\le1}\vert \overline{\Ecal}}\nonumber \\
& \le \PP\rbr{\overline{\Ecal}}\sum_{\ell=1}^K \sum_{k=\ell+1}^K\sum_{(i,j) \in [n]^2}(\lambda_k-\lambda_\ell) \nonumber\\
& = n^2\PP\rbr{\overline{\Ecal}}\sum_{\ell=1}^K \sum_{k=\ell+1}^K(\lambda_k-\lambda_\ell) \nonumber\\
& = n^2\PP\rbr{\overline{\Ecal}}\rbr{\sum_{k=1}^K (k-1)\lambda_k - \sum_{\ell=1}^K (K-\ell)\lambda_\ell} \nonumber\\
& \le n^2K\PP\rbr{\overline{\Ecal}}\sum_{k=1}^K \lambda_k \nonumber\\
& \le 4K \EE[C_{\OPT}]  \PP\left(\overline{\Ecal}\right) \tag{\cref{eq: optimal cost lower bound}} \\
    &\leq \frac{2}{n}\EE[C_{\OPT}], \label{eq:etcu term ii}
\end{align}
where the last inequality is by \cref{eq:concentration_good_event_ETCU}.

\paragraph{Bounding $(i)$.} 
Consider any couple $(k,\ell)\in [K]^2$ s.t. $\ell\leq k$. Let $m^*_{\ell,k}$ be the number of comparisons performed between jobs of type $\ell$ and $k$ before the algorithm detects that $\lambda_\ell \leq \lambda_k$. 
A first obvious upper bound is $m^*_{\ell,k} \leq n$. A second upper is obtained by noting that $m^*_{\ell,k}$ is smaller than any $m'$ s.t. 
\[
\delta^{(m',n)}<\frac{1}{2}\bigg|\frac{\lambda_k}{\lambda_k+\lambda_\ell}-0.5\bigg|.
\]
For this value of $\delta^{(m',n)}$, the event $\Ecal$ ensures that if $\lambda_k\ge \lambda_\ell$, then
\begin{align*}
    \hat{r}^{m'}_{\ell,k} - \delta^{(m',n)} 
    \overset{\text{Under } \Ecal}{\ge} \EE[r^{m'}_{\ell,k}] - 2\delta^{(m',n)}
    > \frac{\lambda_k}{\lambda_k+\lambda_\ell} - \left(\frac{\lambda_k}{\lambda_k+\lambda_\ell}-\frac{1}{2}\right) = \frac{1}{2},
\end{align*}
and type $k$ would be eliminated. This implies the following upper bound on $m^*_{\ell, k}$:
\begin{equation}\label{eq:boundm_ETCUk}
    m^*_{\ell,k}\leq \min\left(n,8\left(\frac{\lambda_k+\lambda_\ell}{\lambda_k-\lambda_\ell}\right)^2\log\left(2n^2K^3\right)\right).
\end{equation}

On the other hand, notice that under the good event $\Ecal$, a type $\ell$ will never be eliminated due to a type $k$ of greater expected duration $\lambda_k\ge \lambda_{\ell}$, since
\begin{align*}
    \hat{r}^{m'}_{k,\ell} - \delta^{(m',n)} 
    \overset{\text{Under } \Ecal}{\le} \left(\EE[r^{m'}_{\ell,k}] +\delta^{(m',n)}\right)  - \delta^{(m',n)}
    =\frac{\lambda_\ell}{\lambda_k+\lambda_\ell} 
    \le \frac{1}{2}.
\end{align*}

We decompose the run of the algorithm into (up to $K$) phases. For each $\ell \in [K]$, we call phase $\ell$ the iterations at which jobs of type $\ell$ are the jobs with the smallest mean still not terminated. Note that during phase $\ell$, job type $\ell$ is always active, as the contrary would mean event $\Ecal$ does not hold. This implies that the number of jobs of any type $k>\ell$ computed during phase $\ell$ is lower than $m^*_{\ell, k}$.
We have the following bound:

\begin{align*}
    (i) & = \sum_{\ell=1}^K \sum_{k=\ell+1}^K\sum_{(i,j) \in [n]^2}( \lambda_k-\lambda_\ell)\EE\sbr{\ind{P_i^k \text{ computed before } P_j^\ell}\vert \Ecal} \nonumber\\
    &\overset{(1)}{\leq} \sum_{\ell=1}^K \sum_{k=\ell+1}^K\sum_{(i,j) \in [n]^2}( \lambda_k-\lambda_\ell)\EE\sbr{\ind{P_i^k \text{ computed before phase } \ell+1}\vert \Ecal} \\ 
    &\leq  \sum_{\ell=1}^K \sum_{k=\ell+1}^K\sum_{i \in [n]}n( \lambda_k-\lambda_\ell)\EE\sbr{\sum_{o=1}^\ell\ind{P_i^k \text{ computed during phase } o}\vert \Ecal} \\ 
    &\overset{(2)}{\leq}  \sum_{\ell=1}^K \sum_{k=\ell+1}^K\sum_{o=1}^\ell \EE [m^*_{o,k}\vert \Ecal]n( \lambda_k-\lambda_\ell) \\ 
    &\leq  \sum_{\ell=1}^K \sum_{k=\ell+1}^K\sum_{o=1}^\ell\EE [m^*_{o,k}\vert \Ecal]n( \lambda_k-\lambda_o) \\
    & \overset{(3)}{=} \sum_{k=1}^K \sum_{o=1}^{k-1}\sum_{l=o}^{k-1}\EE [m^*_{o,k}\vert \Ecal]n( \lambda_k-\lambda_o) \\
    & =  \sum_{k=1}^K\sum_{o=1}^{k-1} \EE [m^*_{o,k}\vert \Ecal]n(k-o)( \lambda_k-\lambda_o) \\
     &\overset{(4)}{\leq} \sum_{k=1}^K\sum_{\ell=1}^{k-1} \EE [m^*_{\ell,k}\vert\Ecal]n(K-\ell)( \lambda_k-\lambda_\ell) \\
    & \overset{(5)}{\leq}  \sum_{k=1}^K \sum_{\ell= 1}^{k-1}  \min\left(n,8\left(\frac{\lambda_k+\lambda_\ell}{\lambda_k-\lambda_\ell}\right)^2\log\left(2n^2K^3\right)\right)n(K-\ell) (\lambda_k - \lambda_\ell). \\
\end{align*}
$(1)$ is since by the beginning of phase $\ell+1$, all jobs of type $\ell$ were completed. $(2)$ is since during phase $o$, the $o^{th}$ type was not eliminated, so there cannot be more than $m^*_{o,k}$ jobs of type $k$ in this phase. In $(3)$, we changed the summation order and in $(4)$, we replaced $o\to\ell$. Finally, $(5)$ is due to the bound of \cref{eq:boundm_ETCUk}, which holds under $\Ecal$.

Next, for any $\lambda_k \geq \lambda_\ell$, we have: 
\[
(\lambda_k - \lambda_\ell) \min\left(n,8\left(\frac{\lambda_k+\lambda_\ell}{\lambda_k-\lambda_\ell}\right)^2\log\left(2n^2K^3\right)\right)\leq  (\lambda_k+ \lambda_\ell) \sqrt{8n \log(2 n^2K^3)},
\]
since $\min\cbr{a,b}\le\sqrt{ab}$ for any $a,b\ge0$.
This implies that
\begin{align*}
	(i) &\le \sum_{k=1}^K \sum_{\ell= 1}^{k-1}  n(K-\ell)(\lambda_k+ \lambda_\ell) \sqrt{8n \log(2 n^2 K^3)}  \\
	&=   \sum_{k=1 }^K\left[\frac{1}{2}(k-1)(2K-k)+(K-k)^2\right]\lambda_k n \sqrt{8n \log(2 n^2 K^3)}.\\
\end{align*}
Substituting this and the bound of \cref{eq:etcu term ii} into the decomposition of \cref{eq:etcu decomposition} gives the first bound of the proposition. 

The second bound is obtained by upper bounding:
\[
(\lambda_k - \lambda_\ell) \min\left(n,8\left(\frac{\lambda_k+\lambda_\ell}{\lambda_k-\lambda_\ell}\right)^2\log\left(2n^2K^3\right)\right)\leq  8\frac{\left(\lambda_k+\lambda_\ell\right)^2}{\lambda_k-\lambda_\ell}\log\left(2n^2K^3\right),
\]
\end{proof}

\subsection{Upper bound for UCB-U} 
\label{sec:ucb-u-analysis}

\begin{proposition}\label{app:bound_UCBU_full}
    The expected cost of $\UCBU$ is upper bounded by:
    \begin{align*}
    &\EE[C_{\UCBU}] 
    \leq \EE[C_{\FTPP}] + n (K-1)\sqrt{3n\ln\rbr{2n^2K^2}} \sum_{k=1}^K\lambda_k + \frac{2}{n}\EE[C_{\OPT}],
    \end{align*}
and:
    \begin{align*}
\EE[C_{\UCBU}] 
    \leq \EE[C_{\FTPP}] + \sum_{\ell=1}^K \sum_{k=\ell+1}^K \frac{(\lambda_k+\lambda_{\ell})^2}{\lambda_k-\lambda_{\ell}}  3n\ln\rbr{2n^2K^2}+ \frac{2}{n} \EE[C_{\OPT}].
\end{align*}
\end{proposition}

\paragraph{Concentration of exponential distribution}
If $X \sim \Ecal(\lambda)$, then $\frac{2}{\lambda}X \sim \Ecal(2) = \chi_2^2$ ($\chi^2$ with $2$ degrees of freedom).
It follows that if $\forall i \in [m], X_i \sim \Ecal(\lambda)$, then $\frac{2}{\lambda} \sum_{i=1}^m X_i \sim \chi^2_{2m}$. Denote $\chi^2_{2m}(\alpha)$ the $\alpha$-th percentile, we have with probability $1 - \delta$ that
\[
\frac{2 \sum_{i} X_i} {\chi^2_{2m}(1 - \delta / 2)} \leq \lambda \leq \frac{2 \sum_{i} X_i} {\chi^2_{2m}(\delta / 2)}
\]

Setting $\delta = \frac{1}{n^2K^2}$, we get the following formula for a lower bound:

\begin{equation}
  \label{app:ucbu:lb}
\underline{\lambda}_k^{m} = \frac{2 \sum_{i=1}^{m} X_i^k}{\chi^2_{2 m}(1 - \frac{1}{2 n^2K^2})}
\end{equation}

and another formula for the upper bound

\begin{equation*}
\overline{\lambda}_k^{m} = \frac{2 \sum_{i=1}^{m} X_i^k}{\chi^2_{2 m}(\frac{1}{2 n^2K^2})}
\end{equation*}

\paragraph{}
If a job $k$ is wrongly scheduled before a job of type $\ell$, then the decision rule is misleading meaning that:
\[
\underline{\lambda}_k^{m_k} = \frac{2 \sum_{i=1}^{n_k} X_i^k}{\chi^2_{2 n_k}(1 - \frac{1}{2 n^2K^2})} < \frac{2 \sum_{i=1}^{n_\ell} X_i^\ell}{\chi^2_{2 n_\ell}(1 - \frac{1}{2 n^2K^2})} = \underline{\lambda}_\ell^{m_\ell}
\]
even though $\lambda_\ell < \lambda_k$.

\paragraph{Bounding the cost}
From \Cref{app:lemma:cost-non-preemptive}, the cost of any non preemptive algorithm $A$ writes
  \begin{align}
    \EE[C_{A}] &= \EE[C_{\FTPP}] + \sum_{\ell=1}^K \sum_{k=\ell+1}^K\sum_{(i,j) \in [n]^2}(  \lambda_k-\lambda_\ell)\EE\sbr{\ind{P_i^k \text{ computed before } P_j^\ell}} \\
  \end{align}

  Let us then introduce the GOOD event which is:
  \[
    \mathcal{E} = \{ \forall i \in [n], \forall k \in [K], \underline{\lambda}_k^{i} \leq \lambda_k \leq \overline{\lambda}_k^{i} \}
  \]
  With a union bound, it is easy to show that $\mathcal{E}$ holds with probability $1 - \frac{1}{nK}$ and that the contradictory event $\overline{\mathcal{E}}$ happens with probability $\frac{1}{nK}$.

Using the same method as in the proof of Proposition~\ref{app:fullboundETCU} (the decomposition using $\mathcal{E}$ and $\overline{\mathcal{E}}$ as done in \Cref{eq:etcu decomposition} and the derivation of \Cref{eq:etcu term ii}), we can upper bound the cost of UCB-U as:
\begin{align*}
    \EE[C_{UCB-U}] &\leq \EE[C_{\FTPP}] + \underbrace{\sum_{\ell=1}^K \sum_{k=\ell+1}^K\sum_{(i,j) \in [n]^2}(  \lambda_k-\lambda_\ell)\EE\sbr{\ind{P_i^k \text{ computed before } P_j^\ell} | \mathcal{E}}}_{(i)} \\& \enspace \enspace \enspace+  \EE[C_{OPT}] \underbrace{4K P(\overline{\mathcal{E}})}_{4 / n}
\end{align*}

Furthermore, $P^k_i$ computed before $P^\ell_j$ implies that  $\underline{\lambda}_k^{i} < \underline{\lambda}_\ell^{j}$ and therefore
\[
(i) \leq \sum_{\ell=1}^K \sum_{k=\ell+1}^K\sum_{(i,j) \in [n]^2}(  \lambda_k-\lambda_\ell)\PP( \underline{\lambda}_k^{i} < \underline{\lambda}_\ell^{j} \vert \mathcal{E})
\]
Under $\Ecal$, we have $\underline{\lambda}_\ell^{j} \le \lambda_\ell$. Moreover, it holds that $\overline{\lambda}_k^{i} \ge \lambda_k$, and by the definition of $\underline{\lambda}_k^{i}$, and $\overline{\lambda}_k^{i}$,
\begin{align*}
     \underline{\lambda}_k^{i} = \frac{ \chi^2_{2i}(\frac{1}{2 n^2K^2})}{\chi^2_{2 i}(1 - \frac{1}{2 n^2K^2})} \overline{\lambda}_k^{i} \ge \frac{ \chi^2_{2i}(\frac{1}{2 n^2K^2})}{\chi^2_{2 i}(1 - \frac{1}{2 n^2K^2})} \lambda_k.
\end{align*}
Combined, under $\Ecal$ we can bound $\cbr{\underline{\lambda}_k^{i} < \underline{\lambda}_\ell^{j}} \subseteq\cbr{\lambda_k \frac{ \chi^2_{2i}(\frac{1}{2 n^2K^2})}{\chi^2_{2 i}(1 - \frac{1}{2 n^2K^2})} < \lambda_{\ell}}$ and therefore write
\begin{align*}
(i) &\leq \sum_{\ell=1}^K \sum_{k=\ell+1}^K\sum_{(i,j) \in [n]^2}\ind{\lambda_k \frac{ \chi^2_{2i}(\frac{1}{2 n^2K^2})}{\chi^2_{2 i}(1 - \frac{1}{2 n^2K^2})} < \lambda_{\ell}} (\lambda_k - \lambda_\ell) \\
&= \sum_{\ell=1}^K \sum_{k=\ell+1}^K n \max \cbr{ i \in [n], \lambda_k \frac{ \chi^2_{2i}(\frac{1}{2 n^2K^2})}{\chi^2_{2 i}(1 - \frac{1}{2 n^2K^2})} < \lambda_{\ell} } (\lambda_k-\lambda_\ell)
\end{align*}

So finally we have
\begin{equation}
  \EE[C_{\UCBU}] \leq \EE[C_{\FTPP}] + \sum_{\ell=1}^K \sum_{k=\ell+1}^K n \max \cbr{ i \in [n], \lambda_k \frac{ \chi^2_{2i}(\frac{1}{2 n^2K^2})}{\chi^2_{2 i}(1 - \frac{1}{2 n^2K^2})} < \lambda_{\ell} } (\lambda_k-\lambda_\ell) + \frac{2}{n} \EE[C_{\OPT}]   \label{eq: ucb-u bound}
\end{equation}

\paragraph{Bounding the ratio.} We now focus on bounding the maximum term in \Cref{eq: ucb-u bound}. 
By Lemma 1 of \citep{laurent2000adaptive}, if $U\sim \chi_D^2$, then
\begin{align}
    \PP\rbr{U\ge D + 2\sqrt{Dx} + 2x}\le \exp(-x),\qquad\text{and}\qquad
     \PP\rbr{U\le D - 2\sqrt{Dx}}\le \exp(-x),
\end{align}
and in particular,
\begin{align}
    \chi_D^2(\delta) \ge D-2\sqrt{D\ln\frac{1}{\delta}},\qquad\text{and}\qquad
     \chi_D^2(1-\delta) \le D+2\sqrt{D\ln\frac{1}{\delta}} + 2\ln\frac{1}{\delta}.
\end{align}
Thus, for any $i\in[n]$, a necessary condition to the inequality $\lambda_k \frac{ \chi^2_{2i}(\frac{1}{2 n^2K^2})}{\chi^2_{2 i}(1 - \frac{1}{2 n^2K^2})} < \lambda_{\ell}$ is
\begin{align*}
    &\frac{2i-2\sqrt{2i\ln\rbr{2n^2K^2}}}{2i+2\sqrt{2i\ln\rbr{2n^2K^2}} + 2\ln\rbr{2n^2K^2}}<\frac{\lambda_{\ell}}{\lambda_k} \\ 
    \Rightarrow &\rbr{1-\frac{\lambda_{\ell}}{\lambda_k}}i-\sqrt{2\ln\rbr{2n^2K^2}}\rbr{1+\frac{\lambda_{\ell}}{\lambda_k}}\sqrt{i}-\frac{\lambda_{\ell}}{\lambda_k}\ln\rbr{2n^2K^2}<0 \\
    \Rightarrow &\rbr{\lambda_k-\lambda_{\ell}}i-\sqrt{2\ln\rbr{2n^2K^2}}\rbr{\lambda_k+\lambda_{\ell}}\sqrt{i}-\lambda_{\ell}\ln\rbr{2n^2K^2}<0 \\
    \Rightarrow &\sqrt{i} <\frac{\sqrt{2\ln\rbr{2n^2K^2}}\rbr{\lambda_k+\lambda_{\ell}} + \sqrt{2\ln\rbr{2n^2K^2}\rbr{\lambda_k+\lambda_{\ell}}^2 +4\ln\rbr{2n^2K^2}\lambda_{\ell}\rbr{\lambda_k-\lambda_{\ell}}}}{2\rbr{\lambda_k-\lambda_{\ell}}} \\
    \Rightarrow &\sqrt{i} <\sqrt{2\ln\rbr{2n^2K^2}}\frac{\rbr{\lambda_k+\lambda_{\ell}} + \sqrt{\rbr{\lambda_k+\lambda_{\ell}}^2 +2\lambda_{\ell}\rbr{\lambda_k-\lambda_{\ell}}}}{2\rbr{\lambda_k-\lambda_{\ell}}}
\end{align*}
Now, using the fact that $2\lambda_\ell\le \lambda_\ell+\lambda_k$, we get the simplified bound 
\begin{align}
    \sqrt{i} <\sqrt{2\ln\rbr{2n^2K^2}}\frac{\rbr{\lambda_k+\lambda_{\ell}} + \sqrt{2\lambda_k\rbr{\lambda_k+\lambda_{\ell}}}}{2\rbr{\lambda_k-\lambda_{\ell}}}
    \le \sqrt{2\ln\rbr{2n^2K^2}}\rbr{1+\sqrt{2}}\frac{\lambda_k+\lambda_{\ell}}{2\rbr{\lambda_k-\lambda_{\ell}}}\enspace,
\end{align}
or $i\le 3\ln\rbr{2n^2K^2}\rbr{\frac{\lambda_k+\lambda_{\ell}}{\lambda_k-\lambda_{\ell}}}^2$. Since we also know that $i\in[n]$, we can write
\begin{align*}
\max \cbr{ i \in [n], \lambda_k \frac{ \chi^2_{2i}(\frac{1}{2 n^2K^2})}{\chi^2_{2 i}(1 - \frac{1}{2 n^2K^2})} < \lambda_{\ell} }
&\leq \min\cbr{3\ln\rbr{2n^2K^2}\rbr{\frac{\lambda_k+\lambda_{\ell}}{\lambda_k-\lambda_{\ell}}}^2,n} \\
& \leq \sqrt{3n\ln\rbr{2n^2K^2}}\frac{\lambda_k+\lambda_{\ell}}{\lambda_k-\lambda_{\ell}},
\end{align*}
where the second inequality is since $\min\cbr{a,b}\leq\sqrt{ab}$ for $a,b>0$. Substituting back into \Cref{eq: ucb-u bound}, we get the first bound in the proposition:
\begin{align*}
    \EE[C_{\UCBU}] 
    &\leq \EE[C_{\FTPP}] + \sum_{\ell=1}^K \sum_{k=\ell+1}^K n \sqrt{3n\ln\rbr{2n^2K^2}}\frac{\lambda_k+\lambda_{\ell}}{\lambda_k-\lambda_{\ell}} (\lambda_k-\lambda_\ell) + \frac{2}{n} \EE[C_{\OPT}]   \\
    &= \EE[C_{\FTPP}] + n \sqrt{3n\ln\rbr{2n^2K^2}}\sum_{\ell=1}^K \sum_{k=\ell+1}^K(\lambda_k+\lambda_{\ell}) + \frac{2}{n} \EE[C_{\OPT}]   \\
    & = \EE[C_{\FTPP}] + n \sqrt{3n\ln\rbr{2n^2K^2}} \rbr{\sum_{k=1}^K(k-1)\lambda_k + \sum_{\ell=1}^K(K-\ell)\lambda_{\ell}} + \frac{2}{n} \EE[C_{\OPT}]   \\
    & = \EE[C_{\FTPP}] + n (K-1)\sqrt{3n\ln\rbr{2n^2K^2}} \sum_{k=1}^K\lambda_k + \frac{2}{n} \EE[C_{\OPT}]  .
\end{align*}
The second bound is obtained through the upper bound:
\[
(\lambda_k-\lambda_\ell)\min\cbr{3\ln\rbr{2n^2K^2}\rbr{\frac{\lambda_k+\lambda_{\ell}}{\lambda_k-\lambda_{\ell}}}^2,n}\leq 3\ln\rbr{2n^2K^2}\rbr{\frac{(\lambda_k+\lambda_{\ell})^2}{\lambda_k-\lambda_{\ell}}}.
\]
\subsection{Lower bounds for Non-Preemptive Algorithms}

\subsubsection{Small Differences (\Cref{prop:dependency in n}) }\label{app: lower bound small diff}

\begin{proof}[Proof of \Cref{prop:dependency in n}]
Assume $K=2$ and take any non-preemptive algorithm $A$. Call $P^1_i$ the $i$-th job of type 1 and $P^2_j$ the $j$-th job of type 2. 
According to \Cref{app:lemma:cost-non-preemptive}, $A$ has a cost
  \[
\EE[C_{A}] = \EE[C_{\FTPP}] + ( \lambda_2-\lambda_1)\EE\sbr{\sum_{(i,j) \in [n]^2} \ind{P_j^2 \text{ computed before } P_i^1}}
  \]
if $\lambda_2 > \lambda_1$ (the role of $\lambda_2$ and $\lambda_1$ are reversed if $\lambda_2 < \lambda_1$).

We then follow the same approach as in chapter 15 in \cite{lattimore_szepesvari_2020}. Consider situation 1 where $\lambda_1 = a, \lambda_2 = b$ and situation 2 where $\lambda_1 = b$ and $\lambda_2 = a$ with $a < b$ and assumes that the adversary chooses the situation based on the algorithm $A$. Intuitively, if $A$ tends to complete more of jobs of type $1$ before jobs of type $2$, the adversary will decide that $\lambda_1 > \lambda_2$ (situation 2) otherwise, it will choose $\lambda_2 > \lambda_1$ (situation 1).
Call $\PP_{\nu_1}$ the joint probability over the scheduling decisions and job sizes in situation 1 following the policy prescribed by algorithm $A$ and $\PP_{\nu_2}$ the same probability in situation 2.
Call $P_{a_t}(x_t)$ the probability that the job of type $a_t$ chosen at time $t$ is of size $x_t$. 
Calling $KL$ the KL divergence, we have following the Lemma 15.1 in~\cite{lattimore_szepesvari_2020}:
$KL(\PP_{\nu_1}, \PP_{\nu_2}) = n (KL(X_a, X_b) + KL(X_b, X_a))$ where $X_a$ is an exponential random variable of expectation $a$ and $X_b$ is an exponential random variable of expectation $b$.

Note right away that $KL(X_a, X_b) = \frac{a}{b} - 1 - \log(\frac{a}{b})$ \citep[\eg][Example 4.2.1]{calin2014geometric},
 therefore $KL(X_a, X_b) + KL(X_b, X_a) =  \frac{a}{b} - 2  + \frac{b}{a} = \frac{(b - a)^2}{ab}$ so 
\[
KL(\PP_{\nu_1}, \PP_{\nu_2}) \leq n \frac{(b - a)^2}{a b}.
\]
The cost of algorithm $A$ in situation 1 is lower bounded as:
\[
\EE_{\nu_1}[C_{A}] \geq \EE_{\nu_1}[C_{\FTPP}] + (b - a) \EE_{\nu_1}\sbr{\ind{\sum_{(i,j) \in [n]^2} \ind{P_j^2 \text{ computed before } P_i^1} \geq n^2 / 2}} n^2 / 2.
\]
The cost of algorithm $A$ in situation 2 is lower bounded as:
\[
\EE_{\nu_1}[C_{A}] \geq \EE_{\nu_2}[C_{\FTPP}] + (b - a) \EE_{\nu_2}\sbr{\ind{\sum_{(i,j) \in [n]^2} \ind{P_i^1 \text{ computed before } P_j^2} > n^2 / 2}} n^2 / 2.
\]
Introduce the event $E = \ind{\sum_{(i,j) \in [n]^2} \ind{P_j^2 \text{ computed before } P_i^1} \geq n^2 / 2}$, we have that
\begin{align*}
\EE[C_A] &= \max_{\nu \in \{ \nu_1, \nu_2 \}} \EE_{\nu}[C_A] \\
&\geq  \frac{\EE_{\nu_1}[C_A] + \EE_{\nu_2}[C_A]}{2} \\
&\geq \frac{\EE_{\nu_1}[C_{\FTPP}] + \EE_{\nu_2}[C_{\FTPP}]}{2} + (b - a) n^2 / 2 \frac{\PP_{\nu_1}(E) + \PP_{\nu_2}(\overline{E})}{2}
\end{align*}
First, let us notice that 
\[
\EE[C_{\FTPP}]= \frac{\EE_{\nu_1}[C_{\FTPP}] + \EE_{\nu_2}[C_{\FTPP}]}{2}
\]
Then, using Bretagnolle–Huber inequality (Th 14.2 in~\cite{lattimore_szepesvari_2020}), we get $\PP_{\nu_1}(E) + \PP_{\nu_2}(\overline{E}) \geq \frac12 \exp(-KL(\PP_{\nu_1}, \PP_{\nu_2}))$ and since 
$KL(\PP_{\nu_1}, \PP_{\nu_2}) \leq n \frac{(\lambda_2 - \lambda_1)^2}{\lambda_1 \lambda_2}$, we have

\[
\EE[C_A] \geq \EE[C_{\FTPP}] + (b - a) n^2 / 2 \frac{\exp(-n\frac{(b - a)^2}{ab})}{4}
\]

At this stage, we can rewrite the equation assuming $\lambda_2 \geq \lambda_1$ and so that we get
\begin{align*}
\EE[C_A] &\geq \EE[C_{\FTPP}] + (\lambda_2 - \lambda_1) n^2 \frac{\exp(-n\frac{(\lambda_2 - \lambda_1)^2}{\lambda_1 \lambda_2})}{8},
\end{align*}
which proves the first result of the proposition. In particular, taking $\lambda_2 \leq \lambda_1\rbr{1 + \frac{1}{\sqrt{n}}}$ gives its second result
\begin{align*}
    \EE[C_A] - \EE[C_{\FTPP}] 
     &\geq \lambda_1 n\sqrt{n}  \frac{\exp(-n\frac{1/n}{(1+1/\sqrt{n})^2})}{8} \\
    &\geq \lambda_1n\sqrt{n} \frac{e^{-1/4}}{8} \\
    & \geq (\lambda_1+\lambda_2)n\sqrt{n} \frac{e^{-1/4}}{24}.
\end{align*}
\end{proof}

\subsubsection{Large Differences}
\label{app: lower bound large diff}
\begin{proposition}
    For any non-preemptive algorithm, there exists a problem instance with expected type durations of $\lambda_1\le\lambda_2\dots\le \lambda_K$ such that  
    \begin{align*}
    \EE[C_A] &\geq \EE[C_{\FTPP}] + \frac{n}{K}\sum_{k=1}^K (2k-K-1)\lambda_k.
\end{align*}
In particular, for $K=2$ and $\lambda_2\ge 3\lambda_1$, it holds that 
\begin{align*}
    \EE[C_A] &\geq \EE[C_{\FTPP}] + \frac{n}{4}(\lambda_1+\lambda_2),
\end{align*}
\end{proposition}
Let $p_k$ be the probability that a non-preemptive algorithm completes a job of type $k$ at its first round. Notice that this distribution cannot depend on the expected duration of any of the types, since no data was gathered. Thus, types can be arbitrarily ordered without affecting this distribution. In particular, we assume w.l.o.g. that $p_1\le p_2\le\dots p_K$ and choose a problem instance where job types are ordered in an increasing duration $\lambda_1\le\lambda_2\le\dots\lambda_K$. Then, the expected cost of the algorithm can be bounded according to \Cref{app:lemma:cost-non-preemptive}, by
\begin{align*}
    \EE[C_{A}] 
    &= \EE[C_{\FTPP}] + \sum_{\ell=1}^K \sum_{k=\ell+1}^K\sum_{(i,j) \in [n]^2}( \lambda_k-\lambda_\ell)\EE\sbr{\ind{P_i^k \text{ computed before } P_j^\ell}} \\
    & \ge \EE[C_{\FTPP}] + \sum_{\ell=1}^K \sum_{k=\ell+1}^K\sum_{j \in [n]}( \lambda_k-\lambda_\ell)\EE\sbr{\ind{P_1^k \text{ computed before } P_j^\ell}} \\
    &\geq \EE[C_{\FTPP}] + n\sum_{\ell=1}^K \sum_{k=\ell+1}^K( \lambda_k-\lambda_\ell)\EE\sbr{\ind{P_1^k \text{ was the first job}}} \\
    & = \EE[C_{\FTPP}] + n\sum_{\ell=1}^K \sum_{k=\ell+1}^K( \lambda_k-\lambda_\ell)p_k \\
    & =  \EE[C_{\FTPP}] + n\sum_{k=1}^K p_k\sum_{\ell=1}^{k-1}( \lambda_k-\lambda_\ell).
\end{align*}
Now, one can be easily convinced that between all probability vectors with non-decreasing components, this bound is minimized by the uniform distribution $p_k=1/K$. To see this, observe that the sum $\sum_{\ell=1}^{k-1}( \lambda_k-\lambda_\ell)$ increases with $k$. Therefore, if $p$ is non-uniform, then $p_K>1/K$, and there would exist a coordinate $k<K$ to which we could move weight from $p_K$, which would decrease the bound. 

Substituting $p_k=1/K$ we then get 
\begin{align*}
    \EE[C_{A}] 
    &\geq  \EE[C_{\FTPP}] + \frac{n}{K}\sum_{k=1}^K\sum_{\ell=1}^{k-1}( \lambda_k-\lambda_\ell)\\
    &= \EE[C_{\FTPP}] + \frac{n}{K}\sum_{k=1}^K (2k-K-1)\lambda_k.
\end{align*}
In particular, if $K=2$, we get 
\begin{align*}
    \EE[C_{A}] 
    &\geq \EE[C_{\FTPP}] + \frac{n}{2}(\lambda_2-\lambda_1)
\end{align*}
and, for $\lambda_2\ge 3\lambda_1$, we have $\lambda_2-\lambda_1\ge \frac{\lambda_1+\lambda_2}{2}$ and thus
\begin{align*}
    \EE[C_{A}] 
    &\geq \EE[C_{\FTPP}] + \frac{n}{4}(\lambda_1+\lambda_2).
\end{align*}


\clearpage

\section{Analysis of Preemptive Learning algorithms}\label{app:preemptive}

\subsection{Full Algorithmic Details}
\label{app: preemptive algs}
In this appendix, we present a full description of $\ETCRR$ and $\UCBRR$.

\begin{algorithm}[H]
\caption{Explore-Then-Commit-Round-Robin ($\ETCRR$)}
\label{algo:etc-rr-k}
\begin{algorithmic}[1]
    \STATE {\bfseries Input :} $n \geq 1$ (number of jobs of each type), $K \geq 2$ (number of types)
    \STATE For all pairs of different types $k, \ell$ initialize $\delta_{k, \ell} = 0$, $\hat{r}_{k, \ell} = 0$ and $h_{k, \ell} = 0$
    \STATE For all types $k$, set $c_{k} = 0$ 
    \REPEAT 
    \STATE $\mathcal{U}$ is the set of types with at least one remaining job
     \IF{$\mathcal{A}$ is empty}
        \STATE $\mathcal{A} = \{ \ell \in \mathcal{U}, \forall k \in \mathcal{U}, k \neq \ell , \enspace \hat{r}_{k, \ell} - \delta_{k, \ell} \leq 0.5 \}$
    \ENDIF
    \STATE Run jobs $ (P^{k}_{c_k+1})_{k \in \mathcal{A}}$ in parallel until a job finishes and denote $\ell$ the type of this job
    \STATE $c_\ell = c_\ell + 1$
    \FOR {$k \in \mathcal{A}, k \neq \ell$}
    \STATE $\beta_{ \ell, k} = \beta_{\ell, k} + 1$
    \STATE $\delta_{ \ell, k} = \delta_{ k, \ell}=\sqrt{\frac{\log(2 n^2K^4)}{2 (\beta_{\ell, k} + \beta_{k,\ell })}}$
    \STATE $\hat{r}_{\ell,k} = \frac{\beta_{ \ell,k}}{\beta_{k, \ell} + \beta_{ \ell, k }}$
    \STATE $\hat{r}_{k,\ell} = \frac{\beta_{ k,\ell}}{\beta_{k, \ell} + \beta_{ \ell, k }}$
    \IF{$\hat{r}_{\ell,k} - \delta_{\ell, k} \geq 0.5$}
    \STATE Remove $k$ from $\mathcal{A}$
    \ENDIF
    \IF{$\hat{r}_{k, \ell} - \delta_{k, \ell} \geq 0.5$ or $c_\ell=n$}
    \STATE Remove $\ell$ from $\mathcal{A}$
    \ENDIF
    \ENDFOR
    \UNTIL{$\mathcal{U}$ is empty}
\end{algorithmic}
\end{algorithm}

\begin{algorithm}[H]
\caption{Upper-Confidence-Bound-Round-Robin ($\UCBRR$)}
\label{algo:ucb-rr-k}
\begin{algorithmic}[1]
    \STATE {\bfseries Input :} $n \geq 1$ (number of jobs of each type), $K \geq 2$ (number of types), discretization step $\Delta$
    \REPEAT 
    \STATE $\mathcal{U}$ is the set of types with at least one remaining job
    \STATE Calculate type indices $u_k$ for all jobs $k\in\Ucal$ according to \cref{eq: ucbrr index}
    \STATE Choose type $\ell\in\argmax_{\ell\in\Ucal}u_{\ell}$
    \STATE Run a job of type $\ell$ for $\Delta$ time units
    \UNTIL{$\mathcal{U}$ is empty}
\end{algorithmic}
\end{algorithm}

\subsection{Cost Decomposition}
\label{app:cost decomp preemptive}
We start with a cost decomposition, which relates the performance of preemptive algorithms to the one of $\FTPP$. We limit ourselves to the natural family of preemptive algorithms that do not simultaneously run two tasks of the same type and is formally defined as follows.
\begin{definition}
Denote $b_i^{\ell}$ and $e_i^{\ell}$ the beginning and end dates of the computation of the $i^{th}$ job of type $\ell$. A \textbf{type-wise non-preemptive algorithm} is an algorithm that computes jobs of the same type one after another, i.e., $\forall i \in [n], \forall k \in [K], e_i^{\ell}\leq b_{i+1}^{\ell}$. 
\end{definition}
This property is very natural, as for exponential durations without knowledge of the real execution times, there is no advantage in simultaneously running two tasks of the same type. Specifically, all of our suggested algorithms fall under this definition.

For such algorithms, the cumulative cost can be compared to $\FTPP$ using the following lemma.
\begin{lemma}[Cost of type-wise non-preemptive algorithms]\label{app:lemma:cost-algo}
  Any type-wise non-preemptive algorithm $A$ has the following upper bound on its cost:
  \[
 \EE[C_{A}] \leq   \EE[C_{\FTPP}]+ \sum_{\ell=1}^K\sum_{k=\ell+1}^K \sum_{(i,j) \in [n]^2}(\lambda_k-\lambda_\ell)\EE\sbr{\ind{e_j^k< b_i^\ell} }+(K-1)n\sum_{\ell=1}^K\lambda_{\ell}.
  \]
\end{lemma}
\begin{proof}
Recall that if $D^A_{ij}$ is the amount of time a job $i$ delays job $j$ when running algorithm $A$, then we can write the cost of algorithm $A$ as
  \[
C_{A} = \sum_{j=1}^N P_j + \sum_{i=1}^N \sum_{j=i+1}^N \rbr{D^A_{ij} + D^A_{ji}}.
  \]
Moreover, if $b_i$,$e_i$ are the start (end) time of job $i$, it always holds that $D^A_{ij}\le P_i\ind{b_i < e_j}$. Using this inequality and dividing the summation into types, we get
\begin{align*}
\EE[C_{A}] \leq & \sum_{\ell=1}^K\sum_{k=\ell+1}^K \sum_{(i,j) \in [n]^2}\left(\EE\sbr{P_i^\ell\ind{b_i^\ell < e_j^k} }+\EE\sbr{P_j^k\ind{b_j^k < e_i^\ell} }\right)\\
& +\sum_{\ell=1}^K \rbr{\sum_{i=1}^n\EE[P_i^\ell]+\sum_{j=i+1}^n \EE[P_i^\ell\ind{b_i^\ell< e_j^\ell}]+\EE[P_j^\ell\ind{b_j^\ell< e_i^\ell}]}.
\end{align*}
Since jobs are independent, the expected duration of a job of type $\ell$ is $\lambda_\ell$, independently of its start time. Also, as the algorithm is type-wise non-preemptive, for all $\ell\in[K], j>i$, we have $\ind{b_j^\ell< e_i^\ell]}=0$ and $\ind{b_i^\ell \leq e_j^\ell]}=1$. Thus, 
\begin{align}
\EE[C_{A}] \leq & \sum_{\ell=1}^K\sum_{k=\ell+1}^K \sum_{(i,j) \in [n]^2}\left(\lambda_\ell\EE\sbr{\ind{b_i^\ell < e_j^k} }+\lambda_k\EE\sbr{\ind{b_j^k< e_i^\ell} }\right) \nonumber\\
&+\sum_{\ell=1}^K \rbr{\sum_{i=1}^n\lambda_\ell+\sum_{j=i+1}^n \lambda_\ell} \nonumber\\
= & \sum_{\ell=1}^K\lambda_\ell \frac{n(n+1)}{2}+\underbrace{\sum_{\ell=1}^K\sum_{k=\ell+1}^K \sum_{(i,j) \in [n]^2}\left(\lambda_\ell\EE\sbr{\ind{b_i^\ell < e_j^k} }+\lambda_k\EE\sbr{\ind{b_j^k < e_i^\ell} }\right)}_{(i)}. \label{eq:rr-decomposition 1}
\end{align}
We can now decompose the event that job $i$ of type $\ell$ started before job $j$ of type $k$ finished:
\begin{align*}
\ind{b_i^\ell < e_j^k}
& = \ind{b_i^\ell < e_j^k \leq e_i^\ell } + \ind{e_i^\ell < e_j^k} \\
& = \ind{b_i^\ell < e_j^k \leq e_i^\ell } + \ind{e_i^\ell < b_j^k} +  \ind{b_j^k \leq e_i^\ell < e_j^k }
\end{align*}
$\cbr{e_i^\ell < b_j^k}$ is the event that job $i$ of type $\ell$ was fully computed before job $j$ of type $k$ started. $\cbr{b_i^\ell < e_j^k \leq e_i^\ell }$ is the event that job $i$ of type $\ell$ was running when job $j$ of type $k$ finished, and reciprocally for $\cbr{b_j^k \leq e_i^\ell < e_j^k }$. This gives the following decomposition:
\begin{align*}
    (i)=& \underbrace{\sum_{\ell=1}^K\sum_{k=\ell+1}^K \sum_{(i,j) \in [n]^2}\left(\lambda_\ell\EE\sbr{\ind{e_i^\ell< b_j^k} }+\lambda_k\EE\sbr{\ind{e_j^k< b_i^\ell} }\right)}_{(ii)}\\
    &+\underbrace{\sum_{\ell=1}^K\sum_{k=\ell+1}^K \sum_{(i,j) \in [n]^2}\lambda_\ell\EE\sbr{\ind{b_i^\ell< e_j^k\leq e_i^\ell} } + \lambda_k \ind{b_j^k\leq e_i^\ell< e_j^k}}_{(iii)} \\
    &+\underbrace{\sum_{\ell=1}^K\sum_{k=\ell+1}^K \sum_{(i,j) \in [n]^2}\lambda_k\EE\sbr{\ind{b_j^k< e_i^\ell\leq e_j^k}+ \lambda_k \ind{b_i^\ell\leq e_j^k< e_i^\ell} }}_{(iv)}
\end{align*}
Since the algorithm is type-wise non-preemptive, a single job of each type may run at a given time. This implies that any job of type $\ell$ cannot be in a middle of two  different jobs of type $k$, namely
\[
\forall (\ell,k) \in [K]^2, \ \ell \neq k, \ \forall (i,j) \in [n]^2,\ \sum_{j=1}^n \ind{b_j^k\leq e_i^\ell< e_j^k} \leq 1.
\]
The same conclusion similarly holds for all other sums in terms $(iii)$ and $(iv)$, and therefore implies the following bound:
\begin{align*}
    (iii)+(iv)\leq  n \sum_{\ell=1}^K\sum_{k=\ell+1}^K (\lambda_\ell+\lambda_k)= (K-1)n\sum_{\ell=1}^K\lambda_{\ell}.
\end{align*}
 We also have $\ind{e_j^k< b_i^\ell}\leq 1- \ind{b_i^\ell< e_j^k}$, thus 
 \begin{align*}
     (ii) 
     &\leq \sum_{\ell=1}^K\sum_{k=\ell+1}^K \sum_{(i,j) \in [n]^2}\rbr{\lambda_\ell+(\lambda_k-\lambda_\ell)\EE\sbr{\ind{e_j^k< b_i^\ell} }} \\
     & = n^2\sum_{\ell=1}^K\sum_{k=\ell+1}^K (K-\ell)\lambda_\ell + \sum_{\ell=1}^K\sum_{k=\ell+1}^K \sum_{(i,j) \in [n]^2}(\lambda_k-\lambda_\ell)\EE\sbr{\ind{e_j^k< b_i^\ell} } .
 \end{align*}
 Combining everything into \Cref{eq:rr-decomposition 1}, we get:
 \begin{align*}
  \EE[C_{A}] & \leq  \sum_{\ell=1}^K\lambda_\ell \frac{n(n+1)}{2} + n^2\sum_{\ell=1}^K\sum_{k=\ell+1}^K (K-\ell)\lambda_\ell \\
  &\quad+ \sum_{\ell=1}^K\sum_{k=\ell+1}^K \sum_{(i,j) \in [n]^2}(\lambda_k-\lambda_\ell)\EE\sbr{\ind{e_j^k< b_i^\ell} }+(K-1)n\sum_{\ell=1}^K\lambda_{\ell}\\
  & = \EE[C_{\FTPP}]+ \sum_{\ell=1}^K\sum_{k=\ell+1}^K \sum_{(i,j) \in [n]^2}(\lambda_k-\lambda_\ell)\EE\sbr{\ind{e_j^k< b_i^\ell} }+(K-1)n\sum_{\ell=1}^K\lambda_{\ell},
 \end{align*}
 where the last equality is by \Cref{lemma:cost FTPP}.
\end{proof}
\subsection{Upper bound for ETC-RR}\label{app:ETC-RR-k-types}

\begin{proposition}
\label{app:bound-ETC-RR}
The following bound holds:
    \begin{align}
	\EE[C_{\ETCRR}]&\leq \EE[C_{\FTPP}]+\frac{12K}{n}\EE[C^{OPT}]  +4n\sqrt{n\log(2n^2K^3)}\sum_{\ell=1}^K (K-\ell)^2\lambda_\ell.
	\end{align}
\end{proposition}

\textbf{Good Event.} For any couple $(k,\ell)$, if at some iteration
$\beta_{k,\ell}+ \beta_{\ell,k} = m$, we define the more precise notation for $\hat{r}_{\ell,k}$ at that iteration as $\hat{r}_{\ell,k}^m$. Notice that $m$ represents the number of times that jobs of either type were completed while both were active. Therefore, $m$ can have $2n$ different values, where the extreme case $m=2n-1$ is, for example, when $n-1$ jobs of type $k$ are first completed and only then all $n$ jobs of type $\ell$ are completed. 

Let us start by showing that the estimators $\hat{r}_{\ell,k}$  well-concentrate around their expectations. Exponential random variables are memory-less, i.e., if $X_i\sim Exp(\lambda_i)$, the law of $X_i$ conditionally on it being larger than a constant is unchanged. In particular, `resetting' (replacing by an independent copy) an exponential random variable at any time that precedes its activation does not affect its distribution. We employ reset when one of the types is completed, or when either of the types is removed from $\Acal$ (and then we discard the comparison between types $k,\ell$), and say that way a comparison is triggered, it is taken from an i.i.d. sequence of comparisons. Specifically, given a deterministic number of comparisons $m$ between types $k,\ell$, we write 
\begin{align*}
    \hat{r}^{m}_{\ell,k}
    &\stackrel{\mathcal{L}}{=} \frac{1}{m}\sum_{i=1}^m\mathds{1}\{X_i^{\ell}<X_i^{k}\},
\end{align*}
with $(X_i^\ell)_{i \in [m]}$ and $(X_i^k)_{i \in [m]}$ independent exponential variables of parameters $\lambda_\ell$ and $\lambda_k$ respectively.

Finally, $\Ecal$ be the event that all comparisons between $k,\ell$ are well-concentrated, namely,
\begin{align*}
\Ecal= \cbr{\forall (k,\ell) \in [K]^2, \forall m \in [2n], \bigg|\hat{r}^m_{\ell,k}-\frac{\lambda_k}{\lambda_\ell+\lambda_k}\bigg|<\delta^{(m,n)}}
\end{align*}
with $\delta^{(m,n)}= \sqrt{\frac{\log(2n^2K^3)}{2m}}$, and recall that for any $m \in [2n]$, $ \EE[\hat{r}^m_{\ell,k}]=\frac{\lambda_k}{\lambda_\ell+\lambda_k}$.
By Lemma \ref{lem:concentration_rm} applied with $\alpha =2$, and a union bound over the $\frac{K(K-1)}{2}$ possible pairs, we have:
\begin{equation}\label{eq:ETCRRboundbadevent}
    \PP(\overline{\Ecal})\leq \frac{1}{nK}.
\end{equation}
Notice that under $\Ecal$ a type $\ell$ will never be eliminated by a type $k$ with $\lambda_k\geq\lambda_\ell$, since 
\begin{align*}
    \hat{r}_{k,\ell} - \delta_{\ell,k} < \rbr{\frac{\lambda_\ell}{\lambda_k+\lambda_\ell}+\delta_{\ell,k}} - \delta_{\ell,k} = \frac{\lambda_\ell}{\lambda_k+\lambda_\ell}\leq\frac{1}{2},
\end{align*}
so if $k$ is the minimal type in $\Acal$, then under the good event, it will never be eliminated. Moreover, type $k$ can only be compared to a type $\ell$ with $\lambda_\ell\leq\lambda_k$ at most 
\begin{align}
    m^\text{max}_{\ell,k}&\leq \min\cbr{2n,8\rbr{\frac{\lambda_k+\lambda_\ell}{\lambda_k-\lambda_\ell}}^2\log(2n^2K^3)}:=m^*_{\ell,k} \label{eq:etcrr max comparisons}
\end{align}
since clearly $m\leq 2n$ and if $m\geq 8\rbr{\frac{\lambda_k+\lambda_\ell}{\lambda_k-\lambda_\ell}}^2\log(2n^2K^3)$, then $\delta_{\ell,k}\leq\frac{1}{4}\frac{\lambda_k-\lambda_\ell}{\lambda_k+\lambda_\ell}$ and
\begin{align*}
    \hat{r}_{\ell,k} - \delta_{\ell,k}
    > \rbr{\frac{\lambda_k}{\lambda_k+\lambda_\ell} - \delta_{\ell,k}} - \delta_{\ell,k}
    \geq \frac{\lambda_k}{\lambda_k+\lambda_\ell} -2\cdot \frac{1}{4}\frac{\lambda_k-\lambda_\ell}{\lambda_k+\lambda_\ell}
    = \frac{1}{2}.
\end{align*}

\textbf{Cost Analysis.} Assume that the active type set $\Acal$ can only change at discrete times $t\in\cbr{0,\Delta,2\Delta,\dots}\triangleq\Tcal$ for some $\Delta>0$. We will later take the limit $\Delta\to0$, which coincides with the following \Cref{algo:etc-rr-k}. 
In the following, we denote by $\Acal(t)$ the active type set at time interval $[t, t+\Delta)$, and $\Ucal(t)$ incomplete type set at $[t, t+\Delta)$. Also, let $b_i^\ell\in\Tcal$ be the start time of the $i^{th}$ job of the $\ell^{th}$ type and $e_i^\ell\in\Tcal$ be its end-time (w.l.o.g., if a task ended at a time $t\notin\Tcal$, we delay its ending to $\ceil{\frac{t}{\Delta}}\Delta$). 

Starting from the cost decomposition of \Cref{app:lemma:cost-algo}, we have 
\begin{align}
     \EE[C_{A}] 
     &\leq   \EE[C_{\FTPP}]+(K-1)n\sum_{\ell=1}^K\lambda_{\ell} 
     + \sum_{\ell=1}^K\sum_{k=\ell+1}^K \sum_{(i,j) \in [n]^2}(\lambda_k-\lambda_\ell)\EE\sbr{\ind{e_j^k< b_i^\ell} } \nonumber\\
     & \leq \EE[C_{\FTPP}]+(K-1)n\sum_{\ell=1}^K\lambda_{\ell}
     + \sum_{\ell=1}^K\sum_{k=\ell+1}^K n(\lambda_k-\lambda_\ell)\underbrace{\sum_{j=1}^n\EE\sbr{\ind{e_j^k< b_n^\ell} }}_{(*)}. \label{eq:etc-rr decomposition}
\end{align}
We focus our attention on bounding term $(*)$. For any $t\in\Tcal$ and every $o\in[K]$, define the events
\begin{align*}
&\Fcal_o(t) = \cbr{ o \in \Acal(t), \forall p<o: p\notin \Acal(t)},\qquad 
\bar{\Fcal}_o(t) = \cbr{\forall p\leq o: p\notin \Acal(t)} \\
\end{align*}
These events capture the notion of \textit{phases}, namely, when $\Fcal_o$ is active, the $o$ is the type of the smallest mean that has not been finished. Then, we can write 
\begin{align*}
(*) &= \sum_{j=1}^n\EE\sbr{\ind{e_j^k< b_n^\ell} } \\
&=  \sum_{j=1}^n\sum_{t\in\Tcal}\EE\sbr{\ind{e_j^k=t+\Delta, e_j^k< b_n^\ell} } \\
&\overset{(1)}{\leq}  \sum_{j=1}^n\sum_{t\in\Tcal}\EE\sbr{\ind{e_j^k=t+\Delta, \ell\in\Ucal(t)} } \\
&\leq \sum_{o=1}^\ell\sum_{j=1}^n \sum_{t\in\Tcal}\EE\sbr{\ind{e_j^k=t+\Delta, \Fcal_o(t)} } + \sum_{j=1}^n \sum_{t\in\Tcal} \EE\sbr{\ind{e_j^k=t+\Delta,\ell\in\Ucal(t), \bar{\Fcal}_\ell(t)} } \\
& \overset{(2)}{\leq} \sum_{o=1}^\ell\sum_{j=1}^n \sum_{t\in\Tcal}\EE\sbr{\ind{e_j^k=t+\Delta, \Fcal_o(t)} } + \sum_{j=1}^n \sum_{t\in\Tcal}\EE\sbr{\ind{e_j^k=t+\Delta,\overline{\Ecal}} }\\
& \overset{(3)}{\leq} \underbrace{\sum_{o=1}^\ell\sum_{j=1}^n \sum_{t\in\Tcal}\EE\sbr{\ind{e_j^k=t+\Delta, \Fcal_o(t)} }}_{(i)} + n\PP(\overline{\Ecal}).
\end{align*}
$(1)$ is true since the event $\cbr{e_j^k=t+\Delta, e_j^k< b_n^\ell}$ implies that $b_n^\ell>t$, so type $\ell$ was not completed at time $t$. $(2)$ holds since under $\Ecal$, the type of the minimal duration expected is never eliminated, so if $\ell\in\Ucal(t)$, it is impossible that $p\notin\Acal(t)$ for all $p\leq \ell$ (either there is a type $p<\ell,p\in\Ucal(t)$ that should not have been eliminated, or type $\ell$ should not have been eliminated since it is still incomplete). Finally, $(3)$ is since every job can only end at one time point.

We now further continue to bound term $(i)$. To do so, observe that if $e_j^k=t+\Delta$, then $k\in\Acal(t)$ and the task $j$ of this type was completed at the interval $[t,t+\Delta)$. Moreover, since the job durations are exponential, the completion of any job in $\Acal(t)$ at interval $[t,t+\Delta)$ is independent of the events that occurred until time $t$. Taking into consideration that only one job of any type can run in every interval, the two following equalities hold:
\begin{align*}
    &\EE\sbr{\ind{e^k_j=t+\Delta, \Fcal_o(t)}} = \rbr{1-\exp(-\Delta/\lambda_k)} \EE\sbr{\ind{\Fcal_o(t),k\in\Acal(t)}} \\
    & \EE\sbr{\ind{e^k_j=t+\Delta \text{ or } e^o_j=t+\Delta, \Fcal_o(t),k\in\Acal(t)}} \\ &\space = \rbr{1-\exp(-\Delta/\lambda_k-\Delta/\lambda_o)} \EE\sbr{\ind{\Fcal_o(t),k\in\Acal(t)}}.
\end{align*}

Thus, $(i)$ can be written as
\begin{align*}
&(i) = \sum_{o=1}^\ell\frac{\rbr{1-\exp(-\Delta/\lambda_k)}}{\rbr{1-\exp(-\Delta/\lambda_k-\Delta/\lambda_o)}}\sum_{j=1}^n \sum_{t\in\Tcal}\EE\sbr{\ind{e^k_j=t+\Delta \text{ or } e^o_j=t+\Delta, \Fcal_o(t),k\in\Acal(t)}} \\
& \leq \underbrace{\sum_{o=1}^\ell\frac{\rbr{1-\exp(-\Delta/\lambda_k)}}{\rbr{1-\exp(-\Delta/\lambda_k-\Delta/\lambda_o)}}\EE\sbr{\sum_{j=1}^n \sum_{t\in\Tcal}\ind{e^k_j=t+\Delta \text{ or } e^o_j=t+\Delta, \Fcal_o(t),k\in\Acal(t),\Ecal}}}_{(ii)} \\
&\quad + \underbrace{\sum_{o=1}^\ell\sum_{j=1}^n \sum_{t\in\Tcal}\EE\sbr{\ind{e^k_j=t+\Delta \text{ or } e^o_j=t+\Delta, \Fcal_o(t),\overline{\Ecal}}}}_{(iii)}
\end{align*}
The bad-event term can be bounded by
\begin{align*}
(iii) &\leq \EE\sbr{\ind{\overline{\Ecal}}\sum_{o=1}^\ell\sum_{j=1}^n \sum_{t\in\Tcal}\ind{e^k_j=t+\Delta \text{ or } e^o_j=t+\Delta, \Fcal_o(t)}} \\
& \leq (\ell+1) n \PP(\overline{\Ecal}).
\end{align*}
For the second term, recall that the $\ell^{th}$ phase, in which type $\ell$ is the smallest one that was not completed, is represented by the event $\Fcal_\ell(t)$. Furthermore, given the good event, the smallest type is never eliminated. so once $\Fcal_\ell(t)$ becomes active, it would end only when all jobs of type $\ell$ are completed. In  other words, the time indices in which $\Fcal_\ell(t)$ hold form a (possibly empty) interval $\Ical_\Delta(\ell)$, which represents the $\ell^{th}$ phase. Thus, term $(ii)$ counts the expected number of times that jobs of either type $k$ or $o$ could finish in this interval while type $k$ is still active.

Now, we take the limit $\Delta\to0$ (and denoting the limit interval $\Ical_\Delta(\ell)\to\Ical(\ell)$). Notice that the limit and expectation are interchangeable by the bounded convergence theorem, as the number of times a job of either type $k$ or $o$ can be completed is bounded by $2n$. 
\begin{align*}
    (ii) 
&\underset{\Delta\to0}{\to} \sum_{o=1}^\ell\frac{\lambda_o}{\lambda_k+\lambda_o}\EE\sbr{\sum_{t\in\Ical(o)}\sum_{j=1}^n \ind{e^k_j\in\Ical(o) \text{ or } e^o_j\in\Ical(o),k\in\Acal(t),\Ecal}}\\
& \leq  \sum_{o=1}^\ell\frac{\lambda_o}{\lambda_k+\lambda_o}m^*_{o,k}.
\end{align*}
The inequality holds since under the good event, at any interval where both types $k$ and $o$ with $\lambda_o\le\lambda_k$ are active, there can be at most $m^*_{o,k}$ comparisons. Substituting $(ii)$ and $(iii)$ back into $(i)$, we get
\begin{align*}
(i) \leq \sum_{o=1}^\ell \frac{\lambda_o}{\lambda_k+\lambda_o}m^*_{o,k} + (\ell+1) n \PP(\overline{\Ecal}),
\end{align*}
and yet again, substituting this, through $(*)$, back into \Cref{eq:etc-rr decomposition}, yields
{\small
\begin{align*}
    \EE[C_{A}] &- \EE[C_{\FTPP}] \\
     & \leq (K-1)n\sum_{\ell=1}^K\lambda_{\ell}
     + \sum_{\ell=1}^K\sum_{k=\ell+1}^K n(\lambda_k-\lambda_\ell)\rbr{\sum_{o=1}^\ell \frac{\lambda_o}{\lambda_k+\lambda_o}m^*_{o,k} + (\ell+2) n\PP(\overline{\Ecal}) } \\
     & \leq (K-1)n\sum_{\ell=1}^K\lambda_{\ell}
     + 2Kn^2\PP(\overline{\Ecal})\sum_{\ell=1}^K\sum_{k=\ell+1}^K (\lambda_k-\lambda_\ell) \\
     &\quad+ \sum_{\ell=1}^K\sum_{k=\ell+1}^K n(\lambda_k-\lambda_\ell)\sum_{o=1}^\ell \frac{\lambda_o}{\lambda_k+\lambda_o}\min\cbr{2n,8\rbr{\frac{\lambda_k+\lambda_o}{\lambda_k-\lambda_o}}^2\log(2n^2K^3)}\tag{\cref{eq:etcrr max comparisons}} \\
     & \leq (K-1)n\sum_{\ell=1}^K\lambda_{\ell}
     + 2K^2n^2\PP(\overline{\Ecal})\sum_{\ell=1}^K\lambda_\ell \\
     &\quad+ \sum_{\ell=1}^K\sum_{k=\ell+1}^K\sum_{o=1}^\ell n(\lambda_k-\lambda_\ell) \frac{\lambda_o}{\lambda_k+\lambda_o}4\sqrt{n\log(2n^2K^3)}\frac{\lambda_k+\lambda_o}{\lambda_k-\lambda_o} \tag{$\min\cbr{a,b}\leq \sqrt{ab},\forall a,b\ge0$} \\
     &\leq \rbr{2\frac{K-1}{n}
     + \frac{4K}{n}}\EE[C^{OPT}] \tag{\cref{eq:ETCRRboundbadevent} and \eqref{eq: optimal cost lower bound}}\\
     &\quad+ 4n\sum_{\ell=1}^K\sum_{k=\ell+1}^K\sum_{o=1}^\ell \lambda_o\sqrt{n\log(2n^2K^3)} \tag{$\lambda_o\le \lambda_\ell$} \\
     & \leq \frac{6K}{n}\EE[C^{OPT}]
     + 4n\sqrt{n\log(2n^2K^3)}\sum_{\ell=1}^K\sum_{o=1}^\ell (K-\ell)\lambda_o \\
     & \leq \frac{6K}{n}\EE[C^{OPT}]
     + 4n\sqrt{n\log(2n^2K^3)}\sum_{\ell=1}^K (K-\ell)^2\lambda_\ell
\end{align*}
\qed
}

\subsection{Analysis of UCB-RR}
\begin{proposition}
\label{prop:bound-UCBRR}
The following bound holds for any $\Delta\leq \frac{\lambda_1}{4}$ and $n\geq \max(20,10\ln(K))$ :
    \begin{align}
	\EE[C_{\UCBRR}]\leq&  \EE[C_{\FTPP}]+\frac{12K}{n}\EE[C^{OPT}]  +6n\sqrt{2n\log(2n^2K^2)}\sum_{\ell=1}^K (K-\ell)\lambda_\ell.
	\end{align}
\end{proposition}
\vspace{-.25cm}

Assume discretization of the time to units of $\Delta$, as was done in the analysis of $\ETCRR$. Specifically, assume that the active job only changes at times $t\in\cbr{0,\Delta,2\Delta,\dots}\triangleq\Tcal$ for some $\Delta>0$. We then denote the index of the discretization step by $h = \frac{t}{\Delta}+1\in\cbr{1,2,\dots}$.
 
For each job type $\ell \in [K]$, we introduce $T_\ell(h)$, the number of times job type $\ell$ has been chosen up to iteration $h$. Due to the fact that job durations are exponential, their increments are independent, and increments of length $\Delta$ of jobs of type $\ell$ have a termination probability of $\mu_\ell=1-e^{-\frac{\Delta}{\lambda_k}}$. Leveraging this, let $(x_\ell^s)_{s\ge1}$ be sequences of i.i.d Bernoulli random variables of mean $\mu_{\ell}$,. We then fix our probability space for the analysis s.t. when choosing a job of type $\ell$ for the $s^{th}$ time, it is terminated if $x_\ell^s=1$. Notice that while we allow the sequence $(x_\ell^s)_{s\ge1}$ to have more than $n$ job terminations, it is of no consequence of the analysis, as the algorithm will never choose a job type after its $n^{th}$ job was terminated.

Next, define the empirical means after running $m$ discretized intervals of type-$\ell$ jobs as $\hat{\mu}_\ell(m):= \frac{1}{m}\sum_{s=1}^{m}x_\ell^s$, and the index at iteration $h$ as
\begin{align} \label{eq: ucbrr index}
u_\ell(h)= \max \left\{\tilde{\mu} \in[0,1]: d\left(\hat{\mu}_\ell(T_\ell(t-1)), \tilde{\mu}\right) \leq \frac{\log\rbr{n^2K^2}}{T_\ell(t-1)}\right\}.
\end{align}
Starting from the cost decomposition of \Cref{app:lemma:cost-algo}, we have 
\begin{align}
     \EE[C_{A}] 
     &\leq   \EE[C_{\FTPP}]+(K-1)n\sum_{\ell=1}^K\lambda_{\ell} 
     + \sum_{\ell=1}^K\sum_{k=\ell+1}^K \sum_{(i,j) \in [n]^2}(\lambda_k-\lambda_\ell)\EE\sbr{\ind{e_j^k< b_i^\ell} } \nonumber\\
     & \leq \EE[C_{\FTPP}]+(K-1)n\sum_{\ell=1}^K\lambda_{\ell}
     + \underbrace{\sum_{\ell=1}^K\sum_{k=\ell+1}^K n(\lambda_k-\lambda_\ell)\sum_{j=1}^n\EE\sbr{\ind{e_j^k< e_n^\ell} }}_{(*)}. \label{eq:ucb-rr decomposition}
\end{align}
Denote $a(h)\in [K]$, the type of job chosen at iteration $h$, and let $\varepsilon_{\ell,k}>0$ be some constant that will be determined later in the proof. Notice that if $e_j^k< e_n^\ell$, then there must be an iteration where type $\ell$ was not completed and the $j^{th}$ job of type $k$ were played and completed:
\begin{align}
(*) \leq & \sum_{h =1}^{\infty}\sum_{\ell=1}^K\sum_{k=\ell+1}^K n(\lambda_k-\lambda_\ell)\EE\sbr{\ind{a(h)=k,\  \ell \in \Ucal(h), x_k^{T_k(h)}=1} }\nonumber\\
 = & \sum_{h=1}^{\infty}\sum_{\ell=1}^K\sum_{k=\ell+1}^K n(\lambda_k-\lambda_\ell)\EE\sbr{\ind{a(h)=k,\  \ell \in \Ucal(h), u_\ell(h)\leq u_k(h),x_k^{T_k(h)}=1} }\nonumber\\
 \leq& \sum_{h=1}^{\infty}\sum_{\ell=1}^K\sum_{k=\ell+1}^K n(\lambda_k-\lambda_\ell)\EE\sbr{\ind{a(h)=k,\  \ell \in \Ucal(h), u_k(h)\geq u_\ell(h)\geq \mu_{\ell}-\varepsilon_{\ell, k},x_k^{T_k(h)}=1} }\nonumber\\
 &+  \sum_{h=1}^{\infty}\sum_{\ell=1}^K\sum_{k=\ell+1}^K n(\lambda_k-\lambda_\ell)\EE\sbr{\ind{a(h)=k,\  \ell \in \Ucal(h), u_\ell(h)\leq \mu_\ell-\varepsilon_{\ell,k},x_k^{T_k(h)}=1} }\nonumber\\
  \overset{(1)}{\leq}& \underbrace{\sum_{h=1}^{\infty}\sum_{\ell=1}^K\sum_{k=\ell+1}^K n(\lambda_k-\lambda_\ell)(1-e^{-\frac{\Delta}{\lambda_k}})\EE\sbr{\ind{a(h)=k, u_k(h)\geq \mu_{\ell}-\varepsilon_{\ell}} }}_{(i)}\nonumber\\
 &+  \underbrace{\sum_{\ell=1}^K\sum_{k=\ell+1}^K n^2(\lambda_k-\lambda_\ell)\PP\rbr{\exists h \text{ s.t. } u_\ell(h)\leq \mu_{\ell}-\varepsilon_{\ell}}}_{(ii)}\label{eq: ucbrr decomposition}
\end{align}
In $(1)$, we get the first line by the memoryless property of exponential random variables, noting that all the events inside the indicator are determined before the beginning of the $h^{th}$ iteration. The second line of this relation uses the fact that all tasks will eventually be completed, so $\sum_{h =1}^{\infty}\EE\sbr{\ind{a(h)=k,x_k^{T_k(h)}=1} }=n$.

\paragraph{Bounding term $(i)$. } We now bound the first term of the decomposition in \cref{eq: ucbrr decomposition}.
\begin{align*}
    (i):=&\sum_{h =1}^{\infty}\sum_{\ell=1}^K\sum_{k=\ell+1}^K n(\lambda_k-\lambda_\ell)(1-e^{-\frac{\Delta}{\lambda_k}})\EE\sbr{\ind{a(h)=k, u_k(h)\geq \mu_{\ell}-\varepsilon_{\ell}} } \\ 
    =&\sum_{\ell=1}^K\sum_{k=\ell+1}^K n(\lambda_k-\lambda_\ell)(1-e^{-\frac{\Delta}{\lambda_k}})\sum_{h =1}^{\infty} \EE\sbr{\ind{a(h)=k, u_k(h)\geq \mu_{\ell}-\varepsilon_{\ell}} }.
\end{align*}
Denoting $\underline{d}(p,q) = d(p,q)\ind{p\leq q}$, we have
$$
\{\mu_\ell-\varepsilon_{\ell,k} \leq u_k(h)\} \implies \{ \hat{\mu}_k(T_k(h))\leq \mu_\ell-\varepsilon_{\ell,k} \text{ and }d(\hat{\mu}_k(T_k(h)),\mu_\ell-\varepsilon_{\ell,k} )\leq \frac{\log\rbr{n^2K^2}}{T_k(h)}\} \text{ or } \{ \hat{\mu}_k(T_k(h))\geq \mu_\ell-\varepsilon_{\ell,k} \} ,
$$
which is equivalent to $\cbr{\underline{d}(\hat{\mu}_k(T_k(h)),\mu_\ell-\varepsilon_{\ell,k} )\leq \frac{\log\rbr{n^2K^2}}{T_k(h)}}$. Thus, we can bound
\begin{align*}
    (i) 
    \leq &\sum_{\ell=1}^K\sum_{k=\ell+1}^K n(\lambda_k-\lambda_\ell)(1-e^{-\frac{\Delta}{\lambda_k}})\sum_{h =1}^{\infty} \EE\sbr{\ind{a(h)=k,\underline{d}(\hat{\mu}_k(T_k(h)),\mu_\ell-\varepsilon_{\ell,k} )\leq \frac{\log\rbr{n^2K^2}}{T_k(h)}} } \\
    \leq & \sum_{\ell=1}^K\sum_{k=\ell+1}^K n(\lambda_k-\lambda_\ell)(1-e^{-\frac{\Delta}{\lambda_k}})\sum_{s =1}^{\infty} \EE\sbr{\underline{d}(\hat{\mu}_k(s),\mu_\ell-\varepsilon_{\ell,k} )\leq \frac{\log\rbr{n^2K^2}}{T_k(h)}} ,
\end{align*}
where the second inequality is since $T_k(h)$ increases by $1$ every time that $a(h)=k$. This can be naturally bounded using the following lemma.
\begin{lemma}  Let $X_1, X_2, \ldots$ be a sequence of Bernoulli independent random variables with mean $\mu$, and let $\hat{\mu}_s=$ $\frac{1}{s} \sum_{t=1}^s X_t$ be the sample mean. Further, let $a>0, \mu'>\mu$ and define $\kappa =\sum_{s=1}^{\infty} \ind{\underline{d}(\hat{\mu}_s,\mu' )\leq \frac{a}{s}}$. Then,
\begin{align*}
\EE[\kappa] 
\leq \inf _{\varepsilon \in(0, \mu'-\mu)}\rbr{\frac{a}{d(\mu+\varepsilon, \mu')}+\frac{1}{d\left(\mu+\varepsilon, \mu\right)}}.
\end{align*}
\end{lemma}
\begin{proof}
The proof closely follows the one of \citep[][Lemma 10.8]{lattimore_szepesvari_2020}. For completeness, we now state the well-known Chernoff bound.
\begin{lemma} [Chernoff's bound, e.g., \citet{lattimore_szepesvari_2020}, Lemma 10.3] \label{lemma: chernoff}
Let $X_1, X_2, \ldots, X_T$ be a sequence of Bernoulli independent random variables with mean $\mu$, and let $\hat{\mu}=$ $\frac{1}{T} \sum_{t=1}^T X_t$ be the sample mean. Then, for $a\geq 0$:
$$
\mathbb{P}(d(\hat{\mu}, \mu) \geq a, \hat{\mu} \leq \mu) \leq \exp (-T a).
$$
\end{lemma}
Let $\epsilon\in(0,\mu'-\mu)$ and $u=\frac{a}{d(\mu+\varepsilon, \mu')}$. Then, it holds that
\begin{align*}
    \EE[\kappa] 
    = &\sum_{s=1}^{\infty} \PP\cbr{\underline{d}(\hat{\mu}_s,\mu' )\leq \frac{a}{s}}\\
    = & \sum_{s=1}^{\infty} \PP\cbr{\hat{\mu}_s\ge\mu'\text{ or } d(\hat{\mu}_s,\mu' )\leq \frac{a}{s}}\\
    \leq & \sum_{s=1}^{\infty} \PP\cbr{\hat{\mu}_s\ge\mu+\varepsilon\text{ or } d(\hat{\mu}_s,\mu' )\leq \frac{a}{s}} \tag{$\mu'>\mu+\epsilon$}\\
    \leq & \sum_{s=1}^{\infty} \PP\cbr{\hat{\mu}_s\ge\mu+\varepsilon\text{ or } d(\mu+\varepsilon,\mu' )\leq \frac{a}{s}}\tag{$d(\cdot,\mu')$ is decreasing in $[0,\mu']$}\\
    \leq & u + \sum_{s=1}^{\infty} \PP\cbr{\hat{\mu}_s\ge\mu+\varepsilon}\\
    \leq & u + \sum_{s=1}^{\infty} \sum_{s=1}^{\infty}  \exp \left(-sd(\mu+\epsilon,\mu)\right) \tag{Chernoff's bound}\\
    \leq & \frac{a}{d(\mu+\varepsilon, \mu')} + \frac{1}{d\left(\mu+\varepsilon, \mu\right)},
\end{align*}
and the proof is concluded by taking the infimum over all $\varepsilon\in(0,\mu'-\mu)$.
\end{proof}
Now, assume w.l.o.g. that $\lambda_k>\lambda_\ell$ (or, equivalently, $\mu_\ell<\mu_k$) for all $k>\ell$; otherwise, terms where $\lambda_k=\lambda_\ell$ in $(i)$ will be equal to $0$. Then, letting $\kappa_{k,\ell}=\sum_{s=1}^{\infty} \ind{\underline{d}(\hat{\mu}_k(s),\mu_\ell-\varepsilon_{\ell,k} )\leq \frac{\log\rbr{n^2K^2}}{s}}$, the last lemma implies that
\begin{align}
    (i) &\leq  \sum_{\ell=1}^K\sum_{k=\ell+1}^K n(\lambda_k-\lambda_\ell)(1-e^{-\frac{\Delta}{\lambda_k}})\mathbb{E}[\kappa_{k,\ell}] \nonumber\\
    & \leq \sum_{\ell=1}^K\sum_{k=\ell+1}^K n(\lambda_k-\lambda_\ell)(1-e^{-\frac{\Delta}{\lambda_k}})\left(\frac{\log\rbr{n^2K^2}}{d(\mu_k+\varepsilon_{k,\ell}, \mu_\ell-\varepsilon_{\ell})}+\frac{1}{d\left(\mu_k+\varepsilon_{k,\ell}, \mu_k\right)}\right) \label{eq:bad_pull_diverge}
\end{align}
for some $\varepsilon_{k,\ell} \in(0, \mu_\ell-\varepsilon_{\ell,k}-\mu_k)$ that will be determined later.

\paragraph{Bounding term $(ii)$. } We next focus on bounding the probabilities at the second summation of \cref{eq: ucbrr decomposition}. To do so, we prove the following lemma, which bounds each of the summands of term $(ii)$.
\begin{lemma}\label{app:boundtaul}
The following bound holds:
$\PP\rbr{\exists h \text{ s.t. } u_\ell(h)\leq \mu_{\ell}-\varepsilon_{\ell,k}}\leq \frac{\mu_\ell}{ n^2 K^2d(\mu_\ell-\varepsilon_{\ell,k},\mu_\ell)}$.
\end{lemma}

\begin{proof}
Define $S_\ell^{\text{max}}:= \sum_{h=1}^{\infty}\ind{a(h)=\ell}$, the number of iterations $\ell$ is picked by the algorithm. We have:
\begin{align*}
    \PP\rbr{\exists h \text{ s.t. } u_\ell(h)\leq \mu_{\ell}-\varepsilon_{\ell,k}}
    =& \PP\rbr{\exists h \text{ s.t. } \hat{\mu}_\ell(T_\ell(h))\leq \mu_{\ell}-\varepsilon_{\ell,k} \text{ and } d(\hat{\mu}_\ell(T_\ell(h)),\mu_{\ell}-\varepsilon_{\ell,k} ) \geq \frac{\log\rbr{n^2K^2}}{T_\ell(h)}}\\
    = & \PP\rbr{\exists s\leq S_\ell^{\text{max}} \text{ s.t. } \hat{\mu}_\ell(s)\leq \mu_{\ell}-\varepsilon_{\ell,k} \text{ and } d(\hat{\mu}_\ell(s),\mu_{\ell}-\varepsilon_{\ell,k} ) \geq \frac{\log\rbr{n^2K^2}}{s}} \\
    \leq & \PP\rbr{\exists 1\leq s < \infty \text{ s.t. } \hat{\mu}_\ell(s)\leq \mu_{\ell}-\varepsilon_{\ell,k} \text{ and } d(\hat{\mu}_\ell(s),\mu_{\ell}-\varepsilon_{\ell,k} ) \geq \frac{\log\rbr{n^2K^2}}{s}}.
\end{align*}

Now, observe that the empirical means $\hat{\mu}_\ell$ decrease in intervals without successes. Namely, if $a<b$ are time indices such that $x_\ell^a=1, x_\ell^b=1$ and for all $s\in[a+1,b-1]$, $x_\ell^s=0$, then for any $s\in [a,b-1]$, it holds that $\hat{\mu}_\ell(s)\geq \hat{\mu}_\ell(b-1)$. We thus have:
\begin{align*}
    \mathbb{P}&\left(\exists 1 \leq s <\infty: d\left(\hat{\mu}_\ell(s), \mu_\ell-\varepsilon_{\ell,k}\right)>\frac{\log\rbr{n^2K^2}}{s},\ \hat{\mu}_\ell(s)\leq \mu_\ell-\varepsilon_{\ell,k}\right) \\
    &\qquad\qquad=\mathbb{P}\left(\exists 1 \leq s <\infty: d\left(\hat{\mu}_\ell(s), \mu_\ell-\varepsilon_{\ell,k}\right)>\frac{\log\rbr{n^2K^2}}{s},\ \hat{\mu}_\ell(s)\leq  \mu_\ell-\varepsilon_{\ell,k},\ x_\ell^{s+1}=1\right).
\end{align*}
Using the union bound, this implies
\begin{align*}
 \PP\rbr{\exists h \text{ s.t. } u_\ell(h)\leq \mu_{\ell}-\varepsilon_{\ell,k}}&\leq \sum_{s=1}^{\infty} \mathbb{P}\left(d\left(\hat{\mu}_\ell(s), \mu_\ell-\varepsilon_{\ell,k}\right)>\frac{\log\rbr{n^2K^2}}{s},\hat{\mu}_\ell(s)\leq \mu_\ell-\varepsilon_{\ell,k}, \text{ and } x_\ell^{s+1}=1\right) \\
& = \sum_{s=1}^{\infty} \mathbb{P}\left( x_\ell^{s+1}=1\right)\mathbb{P}\left(d\left(\hat{\mu}_\ell(s), \mu_\ell-\varepsilon_{\ell,k}\right)>\frac{\log\rbr{n^2K^2}}{s},\hat{\mu}_\ell(s)\leq \mu_\ell-\varepsilon_{\ell,k} \vert x_\ell^{s+1}=1\right)\\
& =  \sum_{s=1}^{\infty}\mu_\ell \mathbb{P}\left(d\left(\hat{\mu}_\ell(s), \mu_\ell-\varepsilon_{\ell,k}\right)>\frac{\log\rbr{n^2K^2}}{s}, \hat{\mu}_\ell(s)<\mu_\ell-\varepsilon_{\ell,k}  \right)\\
& \leq \sum_{s=1}^{\infty}\mu_\ell  \mathbb{P}\left(d\left(\hat{\mu}_\ell(s), \mu\right)>\frac{\log\rbr{n^2K^2}}{s}+d(\mu_\ell-\varepsilon_{\ell,k},\mu_\ell) , \hat{\mu}_\ell(s)<\mu_\ell\right),
\end{align*}
where we used the fact that the sequence $x_\ell^s$ is independent and the last inequality is by \citep[Lemma 10.2, (c)]{lattimore_szepesvari_2020}. Next, using Chernoff's bound (\Cref{lemma: chernoff}), we get
\begin{align*}
 \PP\rbr{\exists h \text{ s.t. } u_\ell(h)\leq \mu_{\ell}-\varepsilon_{\ell,k}}& \leq \mu_\ell \sum_{s=1}^{\infty} \exp \left(-s\left(d(\mu_\ell-\varepsilon_{\ell,k},\mu_\ell)+\frac{\log\rbr{n^2K^2}}{s}\right)\right) \\
& \leq \frac{\mu_\ell}{n^2K^2} \sum_{s=1}^{\infty}  \exp \left(-sd(\mu_\ell-\varepsilon_{\ell,k},\mu_\ell)\right) \\
& \leq \frac{\mu_\ell}{ n^2K^2 d(\mu_\ell-\varepsilon_{\ell,k},\mu_\ell)} ,
\end{align*}
which concludes the proof of \Cref{app:boundtaul}.
\end{proof}
Finally, substituting back into $(ii)$ leads to the bound 
\begin{align}
(ii) 
\leq &\sum_{\ell=1}^K\sum_{k=\ell+1}^K n^2(\lambda_k-\lambda_\ell)\frac{\mu_\ell}{ n^2K^2 d(\mu_\ell-\varepsilon_{\ell,k},\mu_\ell)} \nonumber\\
= & \frac{1}{ K^2 }\sum_{\ell=1}^K\sum_{k=\ell+1}^K \frac{\mu_\ell(\lambda_k-\lambda_\ell)}{d(\mu_\ell-\varepsilon_{\ell,k},\mu_\ell)} . \label{eq: ucbrr term ii bound}
\end{align}

\paragraph{Combining both bounds.} 
\begin{align*}
    (*) \leq & \sum_{\ell=1}^K\sum_{k=\ell+1}^K n(\lambda_k-\lambda_\ell)\left(\frac{\mu_k\log\rbr{n^2K^2}}{d(\mu_k+\varepsilon_{k,\ell}, \mu_\ell-\varepsilon_{\ell,k})}+\frac{\mu_k}{d\left(\mu_k+\varepsilon_{k,\ell}, \mu_k\right)}\right)\\
    &+\sum_{\ell=1}^K\sum_{k=\ell+1}^K (\lambda_k-\lambda_\ell)\frac{\mu_\ell}{  K^2d(\mu_\ell-\varepsilon_{\ell,k},\mu_\ell)}.
\end{align*}
We now use a local refinement of Pinsker's inequality \cite{doi:10.1287/moor.2017.0928}:
\[
d(p,q)\geq \frac{1}{2\max(p,q)}(p-q)^2.
\]
This implies:
\begin{align*}
    (*) \leq & \sum_{\ell=1}^K\sum_{k=\ell+1}^K n(\lambda_k-\lambda_\ell)\left(\frac{2\mu_k\left(\mu_\ell-\varepsilon_{\ell,k}\right)\log\rbr{n^2K^2}}{( \mu_\ell-\varepsilon_{\ell,k}-\mu_k-\varepsilon_{k,\ell})^2}+\frac{2\mu_k(\mu_k+\varepsilon_{k,\ell})}{\varepsilon_{k,\ell}^2}\right)\\
    &+\sum_{\ell=1}^K\sum_{k=\ell+1}^K (\lambda_k-\lambda_\ell)\frac{2\mu_\ell^2}{ K^2 \varepsilon_{\ell,k}^2}.
\end{align*}
\textbf{Case 1}: Assume $\mu_\ell\geq 5\mu_k$. Setting $\varepsilon_{k,\ell}= \mu_k$,   we obtain,
\begin{align*}
    (*) \leq & \sum_{\ell=1}^K\sum_{k=\ell+1}^K n(\lambda_k-\lambda_\ell)\left(\frac{2\mu_k(\mu_\ell-\varepsilon_{\ell,k})\log\rbr{n^2K^3}}{(\mu_\ell-\varepsilon_{\ell,k}-2\mu_k)^2}+4\right)+\sum_{\ell=1}^K\sum_{k=\ell+1}^K (\lambda_k-\lambda_\ell)\frac{2\mu_\ell^2}{ K^2 \varepsilon_{\ell,k}^2}\\
    \leq & \sum_{\ell=1}^K\sum_{k=\ell+1}^K n(\lambda_k-\lambda_\ell)\left(\frac{2\mu_k(\mu_\ell-\varepsilon_{\ell,k})\log\rbr{n^2K^3}}{(\frac{3}{5}\mu_\ell-\varepsilon_{\ell,k})^2}+4\right)+\sum_{\ell=1}^K\sum_{k=\ell+1}^K (\lambda_k-\lambda_\ell)\frac{2\mu_\ell^2}{ K^2 \varepsilon_{\ell,k}^2}
\end{align*}
Setting $\varepsilon_{\ell,k}=\frac{1}{5}\mu_\ell$, we get:
\begin{align*}
    (*) \leq & \sum_{\ell=1}^K\sum_{k=\ell+1}^K n(\lambda_k-\lambda_\ell)\left(\frac{10\mu_k\log\rbr{n^2K^2}}{\mu_\ell}+4\right)+\sum_{\ell=1}^K\sum_{k=\ell+1}^K (\lambda_k-\lambda_\ell)\frac{50}{ K^2}.
\end{align*}
We have $\mu_k= 1-e^{-\frac{\Delta}{\lambda_k}}\leq \frac{\Delta}{\lambda_k}$, and if $\Delta\leq \frac{1}{4}\lambda_\ell$, $\frac{1}{\mu_\ell}\leq 1.13\frac{\lambda_\ell}{\Delta}$, this implies:
\begin{align*}
    (*) \leq & \sum_{\ell=1}^K\sum_{k=\ell+1}^K n\lambda_\ell
    11.3\log\rbr{n^2K^2}+\sum_{\ell=1}^K\lambda_\ell\left(\frac{50}{ K }+4nK\right).
\end{align*}
Since $K\geq 2$, this implies:
\begin{align*}
    (*) \leq & \sum_{\ell=1}^K\sum_{k=\ell+1}^K n\lambda_\ell
    11.3\log\rbr{n^2K^2}+\sum_{\ell=1}^K\lambda_\ell\left(12.5+4n\right)K.
\end{align*}
\textbf{Case 2}: Assume $\mu_\ell\leq 5 \mu_k$. Setting $\varepsilon_{k,\ell}= (\mu_\ell-\mu_k)/4$ and $\varepsilon_{\ell,k}= (\mu_\ell-\mu_k)/4$, we obtain,
\begin{align*}
    (*) \leq & \sum_{\ell=1}^K\sum_{k=\ell+1}^K n(\lambda_k-\lambda_\ell)\frac{\mu_k^2}{(\mu_\ell-\mu_k)^2}  \left(32\log\rbr{n^2K^2}+64\right)+\sum_{\ell=1}^K\sum_{k=\ell+1}^K (\lambda_k-\lambda_\ell)\frac{32\mu_\ell^2}{K^2(\mu_\ell-\mu_k)^2}.
\end{align*}
It also holds that $(*)\leq \sum_{\ell=1}^K\sum_{k=\ell+1}^K n^2(\lambda_k-\lambda_\ell)$. Thus:
\begin{align*}
    (*) \leq & \sum_{\ell=1}^K\sum_{k=\ell+1}^K n(\lambda_k-\lambda_\ell)\frac{4\mu_k}{(\mu_\ell-\mu_k)} \sqrt{2n\left(\log\rbr{n^2K^3}+4\right)}+\sum_{\ell=1}^K\sum_{k=\ell+1}^K (\lambda_k-\lambda_\ell)\frac{4\sqrt{2}\mu_\ell n}{K(\mu_\ell-\mu_k)}.
\end{align*}
If $\Delta \leq \frac{1}{4}\lambda_\ell$, we have:
$$
\frac{1}{\mu_\ell-\mu_k}\leq 1.46 \frac{\lambda_k\lambda_\ell}{(\lambda_k-\lambda_\ell)}.
$$
We thus obtain:
\begin{align*}
    (*) \leq & \sum_{\ell=1}^K\sum_{k=\ell+1}^K 6n\lambda_\ell \sqrt{2n\left(\log\rbr{n^2K^2}+4\right)}+\sum_{\ell=1}^K9n\lambda_\ell.
\end{align*}
For any $n\geq \max(10,10\log(K))$, we have:
\[
\ln(n^2K^2)\leq \frac{1}{2}n
\]
which implies $6n\lambda_\ell \sqrt{2n\left(\log\rbr{n^2K^2}+4\right)}\geq 11.3\log\rbr{n^2K^2}$.
Thus for any $n\geq \max(10,10\log(K))$ and $\Delta \leq \frac{1}{4}\lambda_\ell$,
\begin{align*}
    (*) \leq & \sum_{\ell=1}^K\sum_{k=\ell+1}^K 6n\lambda_\ell \sqrt{2n\left(\log\rbr{n^2K^2}+4\right)}+\frac{10K}{n}\EE[C_{OPT}].
\end{align*}

\clearpage

\section{Additional experiments}\label{app:comparisonLSPET}
We implemented the Bayesian approach of~\cite{marban2011learning} that we call LSEPT. We used an uninformative prior $\alpha=1$, $w=0$ (the same for all job types). LSEPT is then in essence a greedy algorithm. Whenever a job finishes, it runs until completion a job whose type has the lowest empirical mean size (computed across jobs that have been processed so far).

We ran all algorithms with $K=2$, where jobs of type $1$ have a mean size $\lambda_1=0.8$ and jobs of type $2$ have a mean $\lambda_2 = 1$. 

As can be seen in Figure~\ref{fig:etc-vs-opt-2jobs-rebuttal}, $\LSEPT$ has better mean performance than $\RRR$, a non-adaptive method. However, it has a large variance and its performance does not improve with $n$. This is typical of the performance of greedy algorithms: since the algorithm commits very early, it can either get very good or very bad performances. We plot the mean over $200$ seeds.

\begin{figure}[H]
  \centering
  \includegraphics[width=.65\textwidth]{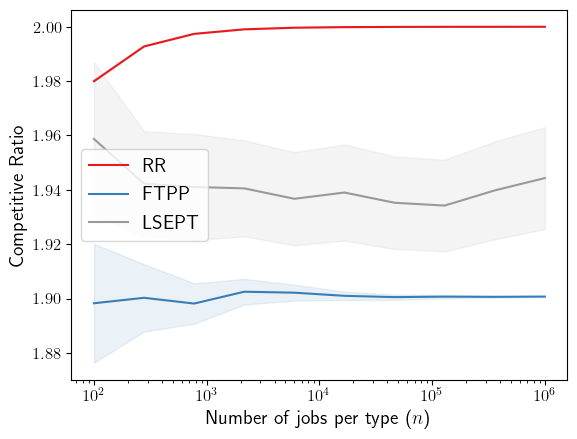}

  \caption{\textbf{CR on jobs with 2 different types}.  $K=2$,  $\lambda_2=1$ and $\lambda_1=0.8$, $n$ takes a grid of values. }
  \label{fig:etc-vs-opt-2jobs-rebuttal}
\end{figure}

\end{document}